\documentclass{article}
\usepackage{graphicx} 
\usepackage{amsmath,amssymb,amsthm}
\usepackage{fullpage}
\usepackage[hcentering]{geometry}
\usepackage{caption} \usepackage{subcaption}
\usepackage{graphicx}
\usepackage{float}
\usepackage{placeins}
\usepackage{paralist}
\usepackage[utf8]{inputenc}   
\usepackage[T1]{fontenc} 
\newif\iffinal
\finaltrue
\newif\ifarxiv
\arxivfalse

\usepackage{paralist}
\usepackage{url}

\usepackage{aliascnt}
\def\NewTheorem#1#2{%
  \newaliascnt{#1}{theorem}
  \newtheorem{#1}[#1]{#2}
  \aliascntresetthe{#1}
  \expandafter\def\csname #1autorefname\endcsname{#2}
}
 \newtheorem{theorem}{Theorem}[section]
\NewTheorem{lemma}{Lemma}
\NewTheorem{corollary}{Corollary}
\NewTheorem{conjecture}{Conjecture}
\NewTheorem{observation}{Observation}
\NewTheorem{definition}{Definition}
\NewTheorem{assumption}{Assumption}

\NewTheorem{proposition}{Proposition}
\NewTheorem{remark}{Remark}
\NewTheorem{claim}{Claim}
\NewTheorem{fact}{Fact}

\newcommand{\hide}[1]{}

\DeclareMathOperator{\lmax}{\ell_{max}}

\usepackage{color}
\usepackage[usenames,dvipsname s,svgnames,table]{xcolor}

\definecolor{darkred}{rgb}{0.5,0,0}
\definecolor{lightblue}{rgb}{0,0.4,0.8}
\definecolor{darkgreen}{rgb}{0,0.5,0}

\usepackage[colorlinks=true, linkcolor=blue, urlcolor=blue, citecolor=darkgreen]{hyperref}

\newcommand{\nnote}[1]{{\color{purple}[N: #1]}}

\usepackage[utf8]{inputenc}

\begin{document}

\title{Multi-Neuron Representations of Hierarchical Concepts \\
in Spiking Neural Networks}
\author{Nancy Lynch\footnote{Massachusetts Institute of Technology.  This work was supported by NSF under grants CCR-2139936 and CCR-2003830.}}
\date{March 21, 2025}
\maketitle

\begin{abstract}
We describe how hierarchical concepts can be represented in three types of layered neural networks.  The aim is to support recognition of the concepts when partial information about the concepts is presented, and also when some of the neurons in the network might fail.  Our failure model involves initial random failures.
The three types of networks are:  feed-forward networks with high connectivity, feed-forward networks with low connectivity, and layered networks with low connectivity and with both forward edges and "lateral" edges within layers.  
In order to achieve fault-tolerance, the representations all use multiple representative neurons for each concept.

We show how recognition can work in all three of these settings, and quantify how the probability of correct recognition depends on several parameters, including the number of representatives and the neuron failure probability.
%
We also discuss how these representations might be learned, in all three types of networks.  For the feed-forward networks, the learning algorithms are similar to ones used in~\cite{LM21}, whereas for networks with lateral edges, the algorithms are generally inspired by work on the assembly calculus~\cite{DBLP:conf/innovations/Legenstein0PV18,PapadimitriouVempala,PVMM20}.
\end{abstract}

\section{Introduction}

We are interested in the general problem of representing hierarchically-structured concepts in layered Spiking Neural Networks (SNNs), in a way that supports concept recognition when partial information is presented.
This work is motivated by computer vision, and by considering how structured concepts might be represented in the brain.

Lynch and Mallmann-Trenn began this work in~\cite{LM21} with a study of simple embeddings of tree-structured concept hierarchies in feed-forward networks.
These embeddings had one representative ($rep$) neuron for each concept, located at a network layer corresponding to the level of the concept in the concept hierarchy.
We focused in~\cite{LM21} on how such embeddings might be learned, using systematic bottom-up presentation of the concepts, and simple Hebbian learning rules.
We also described how the learned representations can be used to recognize hierarchical concepts, given partial information about the concepts.

We continued in~\cite{DBLP:conf/sirocco/LynchM23} with single-$rep$ representations in networks that included feedback edges, and considered a generalization of the tree-structured hierarchies to allow a small amount of overlap between sets of child concepts.
In~\cite{DBLP:conf/sirocco/LynchM23}, we focused on recognition.

In this paper, we continue this effort by studying representations that use multiple $rep$ neurons to represent each concept.
We understand that realistic models of the brain should not use only a single $rep$ neuron per concept.
Real brains use groups of neurons, sometimes called \emph{assemblies}, to represent concepts;
see, for example, \cite{DBLP:conf/innovations/Legenstein0PV18,PapadimitriouVempala, PVMM20} for  neuroscience background on assemblies as well as theoretical models for assembly formation.
The main advantage of multiple neurons is \emph{fault-tolerance}:  multiple neurons provide redundancy so that the memory of a concept can survive some neuron failures.

Thus, we consider representations of hierarchically-structured concepts in layered SNNs, using multiple $rep$ neurons for each concept. 
We continue to represent a concept hierarchy using an embedding as in~\cite{LM21}, only now each concept has a fixed number $m > 1$ of $reps$, at a network layer corresponding to the level of the concept.  
Our intention is that using multiple $reps$ should provide some fault-tolerance during the recognition process.
We consider three types of layered SNNs:  feed-forward networks with high connectivity, feed-forward networks with low connectivity, and layered networks with low connectivity and with both feed-forward edges and edges within layers.
Our failure model involves initial random neuron failures.
Specifically, we assume that each neuron fails initially, with a small probability $q$.
We assume that failures of different neurons are independent.

In this paper, we focus mainly on concept recognition when partial information is presented.
We describe how recognition can work in all three types of networks, and quantify how the probability of correct recognition depends on several parameters, including the number $m$ of $reps$, and the neuron failure probability $q$.
As one might expect, the probability of correct recognition increases with an increase in the number of $reps$ and decreases with an increase in the failure probability.  
The proofs use Chernoff and union bounds.

The paper also discusses informally how these representations might be learned, in all three types of networks.  For the feed-forward networks, the learning algorithms are similar to ones used in~\cite{LM21}, whereas for networks with lateral edges, the algorithms are generally inspired by work on assembly calculus~\cite{DBLP:conf/innovations/Legenstein0PV18,PapadimitriouVempala,PVMM20}.

\paragraph{Roadmap:}
We define the concept hierarchy model in Section~\ref{sec: concepts}, along with a notion of "support" that we use for describing partial information about concepts.
We then describe the general SNN modeling framework in Section~\ref{sec: networks}, and
define the recognition problem in Section~\ref{sec: prob-recog}.

Sections~\ref{sec: recog-ff} begins our work on fault-tolerant recognition.
Here we consider recognition in feed-forward networks with high connectivity.
Namely, we use the simple assumption that all $m$ of the $reps$ of every child of a concept $c$ are connected to all the $reps$ of $c$, using edges with weight $1$.
This section contains results saying that recognition with partial information works correctly in feed-forward networks with high connectivity, even if some randomly-chosen neurons fail.
Our main result of this section describes some conditions under which recognition will occur, with high probability that depends on the number $m$ of $reps$ and the probability $q$ of failure of individual neurons.
We contrast this result with another result giving conditions under which recognition is guaranteed not to occur.

Section~\ref{sec: recog-ff-low} extends these result to feed-forward networks with lower connectivity, expressed by saying that, for each $rep$ $v$ of a concept $c$, and for every child $c'$ of $c$, at least a certain fraction $a m$ of the $reps$ of $c'$ are connected to $v$.   
%
Section~\ref{sec: recog-lat-low} extends the results further, this time for networks with low connectivity that include some lateral edges.
The presence of lateral edges allows us to weaken the connectivity assumption to some extent, compared to Section~\ref{sec: recog-ff-low}, allowing somewhat fewer edges from children's $reps$ to parents' $reps$, but compensating for the missing edges with lateral edges.
This approach is inspired by work on the assembly calculus~\cite{DBLP:conf/innovations/Legenstein0PV18,PapadimitriouVempala,PVMM20}.

Sections~\ref{sec: learning} and~\ref{sec: learning-lateral} discuss informally how the representations of this paper might be learned, in feed-forward networks and networks with some lateral edges, respectively.
For feed-forward networks, the learning algorithms are similar to ones used in~\cite{LM21}, whereas for networks with lateral edges, the algorithms are generally inspired by work on assembly calculus~\cite{DBLP:conf/innovations/Legenstein0PV18,PapadimitriouVempala,PVMM20}.
We conclude in Section~\ref{sec: conclusions} with some discussion and suggestions for future work.

\section{Concept Model}
\label{sec: concepts}

We define concept hierarchies as in~\cite{LM21}.  
In general, we think of a concept hierarchy as containing all the concepts that have been learned by an organism over its lifetime.
However, when we consider recognition of concepts later in the paper, we will focus on a particular concept within the hierarchy, and not worry about the entire hierarchy.


\subsection{Preliminaries}

Throughout the paper, in referring to concept hierarchies, we use the following parameters:
\begin{itemize}
\item $\ell_{max}$, a positive integer, representing the maximum level number for the concepts that we consider.
\item $n$, a positive integer, representing the total number of lowest-level concepts that we consider.
\item $k$, a positive integer, representing the number of top-level concepts in any concept hierarchy, and also the number of child concepts for each concept whose level is $\geq 1$.\footnote{Using the same value of $k$ everywhere is a simplification, which we have made in order to simplify the analysis}
\item $r_1, r_2$, reals in $[0,1]$, with $r_1 \leq r_2$; these represent thresholds for noisy recognition.
\end{itemize}

We assume a universal set $D$ of \emph{concepts}, partitioned into disjoint sets $D_{\ell}$,  $0 \leq \ell \leq \lmax$.
We refer to any particular concept $c \in D_{\ell}$ as a \emph{level} $\ell$ \emph{concept}, and write $level(c) = \ell$.
Here, $D_0$ represents the most basic concepts and $D_{\lmax}$ the highest-level concepts.
We assume that $|D_0| = n$.

\subsection{Concept hierarchies}

A \emph{concept hierarchy} $\mathcal C$ consists of a subset $C$ of $D$, together with a $children$ function.  For each $\ell$, $0 \leq \ell \leq \lmax$, we define $C_{\ell}$ to be $C \cap D_{\ell}$, that is, the set of level $\ell$ concepts in $\mathcal C$.
For each concept $c \in C_{\ell}$, $1 \leq \ell \leq \ell_{max}$, we designate a nonempty set $children(c) \subseteq C_{\ell-1}$.
We call each $c' \in children(c)$ a \emph{child} of $c$.
We assume the following properties.

\begin{enumerate}
\item
$|C_{\lmax}| = k$; that is, the number of top-level concepts is exactly $k$.
\item
For any $c \in C_{\ell}$, where $1 \leq \ell \leq \lmax$, we have that $|children(c)| = k$; that is, the degree of any internal node in the concept hierarchy is exactly $k$.
\item
For any two distinct concepts $c$ and $c'$ in $C_{\ell}$, where $1 \leq \ell \leq \lmax$, we have that $children(c) \cap children(c') = \emptyset$; that is, the sets of children of different concepts at the same level are disjoint.\footnote{Thus, we allow no overlap between the sets of children of different concepts. We study overlap in~\cite{DBLP:conf/sirocco/LynchM23}.}
\end{enumerate}
Thus, a concept hierarchy ${\mathcal C}$ is a forest with $k$ roots and height $\lmax$.  Of course, this is a drastic simplification of any real concept hierarchy, but the uniform structure makes the algorithms easier to analyze.

We extend the $children$ notation recursively by defining a concept $c'$ to be a $descendant$ of a concept $c$ if either $c' = c$, or $c'$ is a child of a descendant of $c$.
%
We write $descendants(c)$ for the set of descendants of $c$.
Let $leaves(c) = descendants(c) \cap C_0$, that is, all the level 0 descendants of $c$.

\subsection{Support}
\label{sec:support}

Now we define which sets of level $0$ concepts should trigger recognition of higher-level concepts.  This is our way of describing partial information about higher-level concepts.

We fix a particular concept hierarchy $\mathcal C$, with its concept set $C$ partitioned into $C_0,\ldots,C_{\lmax}$.
For any given subset $B$ of the universal set $D_0$ of level $0$ concepts, and any real number $r \in [0,1]$, we define the set $supp_r(B)$ of concepts in $C$.
This is intended to represent the set of concepts $c \in C$, at all levels, that have enough of their leaves present in $B$ to support recognition of $c$. 
The notion of "enough" here is defined recursively, based on having an $r$-fraction of children supported for every concept at every level.

\FloatBarrier
\begin{definition}[\textbf{Supported}]
\label{def:support}
Given $B \subseteq D_0$, define the following sets of concepts at all levels, recursively:
\begin{enumerate}
\item
$B(0) = B \cap C_0$. 
\item
For $1 \leq \ell \leq \ell_{max}$, $B(\ell)$ is the set of all concepts $c \in C_{\ell}$ such that $|children(c) \cap B(\ell - 1)|  \geq r k$. 
\end{enumerate}
Define $supp_r(B)$ to be $\bigcup_{0 \leq \ell \leq \lmax} B(\ell)$.  
\end{definition}

\section{Layered Network Model}
\label{sec: networks}

In this paper, we consider two types of layered SNNs:  feed-forward networks, in which all edges point from neurons in one layer to neurons in the next-higher layer, and networks that, in addition to forward edges, include "lateral" edges between neurons in the same layer, for layers $\geq 1$.
The lateral edges are inspired by work on the assembly calculus~\cite{DBLP:conf/innovations/Legenstein0PV18,PapadimitriouVempala,PVMM20}; our layers correspond to "areas" in those papers.
%

We will also consider neuron failures during the recognition process.
For our failure model, we consider initial stopping failures: if a neuron fails, it never performs any activity, that is, it never updates its state and never fires. 

\subsection{Preliminaries}

Throughout the paper, in referring to our networks, we use the following parameters:
\begin{itemize}
     \item
     $\ell'_{max}$, a positive integer, representing the maximum number of a layer in the network.
     \item
     $n$, a positive integer, representing the number of distinct level $0$ concepts the network can handle;  This is the same as $n$ in the concept model, where it represents the total number of level $0$ concepts to be handled.
     \item 
     $m$, a positive integer, representing the number of $reps$ we will assume for each concept.
     \item
     $\tau$, a real number, representing the firing threshold for neurons.
\end{itemize}

\subsection{Network structure}

Our networks are directed graphs consisting of neurons arranged in layers, with forward edges directed from each layer to the next-higher layer.  In our feed-forward networks, those are the only edges.  In our networks with lateral edges, we also have edges from neurons in a layer to other neurons in the same layer.

Specifically, a network $\mathcal{N}$ consists of a set $N$ of neurons, partitioned into disjoint sets $N_{\ell}, 0 \leq \ell \leq \ell'_{max}$, which we call \emph{layers}.
We refer to any particular neuron $u \in N_{\ell}$ as a \emph{layer} $\ell$ \emph{neuron}, and write $layer(u) = \ell$.
We assume that each layer contains exactly $n \cdot m$ neurons, that is, $|N_\ell |= n \cdot m$ for every $\ell$.
We refer to the $n \cdot m$ layer $0$ neurons as \emph{input neurons}.

We assume total connectivity between successive layers, that is, each neuron in $N_{\ell}$, $0 \leq \ell \leq \ell'_{max} - 1$, has an outgoing edge to each neuron in  $N_{\ell+1}$.  In our feed-forward networks, these are the only edges.
In our networks with lateral edges, we also assume total connectivity within each layer, that is, each neuron in $N_{\ell}$, $1 \leq \ell \leq \ell'_{max}$, has an outgoing edge to each other neuron in the same set $N_{\ell}$.

Corresponding to each concept $c \in D_0$, we assume a designated size-$m$ set $reps(c)$ of neurons in $N_0$.  These sets are disjoint for different $c$.  These are the input representations of $c$.
For any $B \subseteq D_0$, we define $reps(B) = \bigcup_{b \in B} reps(b)$.
That is, $reps(B)$ is the set of all $reps$ of concepts in $B$.
%

In later sections, we will consider extensions of the $reps$ function from level $0$ concepts to higher-level concepts.  Establishing such higher-level $reps$ will be the responsibility of a learning algorithm.

\subsection{Neuron states}

We assume that the state of each neuron consists of several state components.
In this paper, all neurons have the following state components:
\begin{itemize}
    \item \emph{failed}, with values in  $\{0,1\}$; this indicates whether or not the neuron has failed.
    \item \emph{firing}, with values in  $\{0,1\}$; this indicates whether or not the neuron is currently firing, where $1$ indicates that it is firing and $0$ indicates that it is not firing.
\end{itemize}

We denote the \emph{firing} component of neuron $u$ at integer time $t$ by $firing^u(t)$ and the \emph{failed} component of neuron $u$ at time $t$ by $failed^u(t)$.  In this paper, failures occur only at the start, that is, each $failed^u$ component is constant over time.  We will maintain the invariant that $failed = 1$ implies $firing = 0$, that is, a failed neuron does not fire.

Each non-input neuron $u \in N_{\ell}$, $1 \leq \ell \leq \ell'_{max}$, has two additional state components:
\begin{itemize}
\item
\emph{weight}, a real-valued vector in $[0,1]^n$ representing the current weights of all incoming edges; these can be edges from neurons at the next-lower level or the same level.
\item
\emph{engaged}, with values in $\{0,1\}$, indicating whether the neuron is currently able to learn new weights.  
\end{itemize}
We denote these two components of non-input neuron $u$ at time $t$ by $weight^u(t)$ and $engaged^u(t)$, respectively.

\subsection{Network operation}
\label{sec: neuron-transitions}

The network operation is determined by the behavior of the individual neurons.
We distinguish between input neurons and non-input (internal and output) neurons.

\paragraph{Input neurons:}
If $u$ is an \emph{input neuron}, then it has only two state components, $failed$ and $firing$.
Since $u$ is an input neuron, we assume that the values of both $failed$ and $firing$ are 
controlled by the network's environment and not by the network itself; that is, the value of $failed^u(t)$ and $firing^u(t)$ are set by some external force, which we do not model explicitly.
Since we assume initial stopping failures, the value of $failed$ is the same at every time $t$, that is, $failed^u(t) = failed^u(0)$ for every $t \geq 0$.
We assume that if an input neuron fails, it never fires, that is, $failed^u(0) = 1$ implies that $firing^u(t) = 0$ for every $t \geq 0$.

\paragraph{Non-input neurons:}
If $u$ is a \emph{non-input neuron}, then it has four state components, $failed$, $firing$, $weight$, and $engaged$.
The value of $failed^u$ is set by an external force, as for input neurons.  Again, since we assume initial stopping failures, the value of $failed$ is the same at every time $t$.
A non-input neuron $u$ that fails never fires, that is, $failed^u(0) = 1$ implies that $firing^u(t) = 0$ for every $t \geq 0$.

For a non-input neuron $u$ that does not fail, the value of $firing^u(0)$ is determined by the initial network setting, whereas the value of $firing^u(t)$, $t \geq 1$, is determined by $u$'s incoming \emph{potential} and its \emph{activation function}.
To define the potential, let $x^u(t)$ denote the vector of $firing$ values of $u$'s incoming neighbor neurons at time $t$.
These are all the nodes in the layer below $u$ plus (for our model with lateral edges) all the nodes in the same layer as $u$. 
Then the potential for time $t$, $pot^u(t)$, is given by the dot product of the $weight$ vector and incoming firing pattern at neuron $u$ at time $t-1$, that is, 
\[
pot^u(t) = weight^u(t-1) \cdot x^u(t-1) = \sum_j weight^u_j(t-1) x^u_j(t-1)j,
\]
where $j$ ranges over the set of incoming neighbors.
The activation function, which defines whether or not neuron $u$ fires at time $t$, is then defined by:
\[ firing^u(t) =  \begin{cases}
1 & \text{if $pot^u(t) \geq \tau$}, \\
0 & \text{otherwise},
\end{cases}\]
where $\tau$ is the assumed firing threshold.

For non-input neurons that fail, we assume that $engaged(t) = 0$ for every $t$.
For a non-failed input neuron $u$, we assume that the value of $engaged^u$ is controlled by an external force, which may arise from outside the network or from another part of the network, such as a "Winner-Take-All" sub-network.

For a non-input neuron $u$, the value of $weight^u(0)$ is determined by the initial network setting.
For a non-input neuron $u$ that fails, the $weight$ vector remains unchanged, $weight^u(t) = weight^u(0)$ for every $t \geq 0$.

A non-input neuron $u$ that does not fail and is engaged at time $t \geq 1$ determines $weight^u(t)$ according to a Hebbian-style learning rule, based on the $weight$ vector and incoming firing pattern at time $t-1$.
In our previous work~\cite{LM21}, we assumed Oja's learning rule.  
That is, if $engaged^u(t) = 1$, then (using vector notation for $weight^u$ and $x^u$):
\begin{equation}\label{eq:Oja} 
    \text{\emph{Oja's rule}:  $weight^u(t) = weight^u(t-1) + \rho \ pot(t) \cdot  (  x^u(t-1) -   pot(t)\cdot weight^u(t-1) )$,} 
\end{equation}
where $\rho$ is an assumed learning rate.
Thus, the weight vector is adjusted by an additive amount that is proportional to the 
potential, and depends on the incoming firing pattern, with a negative adjustment that depends on the potential and the prior weights.\footnote{
However, in Sections~\ref{sec: learning} and~\ref{sec: learning-lateral}, we will require variations on this rule.  In particular, we would like our learning rules to achieve exactly $1$ and exactly $0$ as the final edge weights, instead of the approximate, scaled versions achieved by the rule above.}
A non-input neuron $u$ that does not fail but is not engaged does not change its weights, that is, $weight^u(t) = weight^u(t-1)$.

During execution, the network proceeds through a sequence of \emph{configurations}, where each configuration specifies a state for every neuron in the network, that is, values for all the state components of every neuron.

\section{Problem Definitions}
\label{sec: prob-recog}

In this section, we define the recognition problem formally, and the learning problem less formally.  
The definitions for recognition include specifications of how the input is presented, and what outputs should be produced, with what probabilities.
The definitions are slightly different for feed-forward networks and networks with lateral edges, because of differences in timing requirements.  

\subsection{Preliminaries}

Our recognition problems use the following new parameters:
\begin{itemize}
\item 
$q$, the failure probability of each neuron; define $p = 1-q$ to be the probability of a neuron not failing, i.e., the probability of a neuron \emph{surviving}.
\item  
For each $\ell$, $0 \leq \ell \leq \ell_{max}$, $\delta_{\ell} \in [0,1]$; these  are the \emph{recognition probability} parameters, representing the probability of insufficient firing during the recognition process.
\item  $\epsilon \in [0,1]$; this is the \emph{recognition approximation} parameter, representing a fraction of $rep$ neurons that might not fire.
\end{itemize}

For all versions of the recognition problem, we assume that the failures of all neurons (including the input neurons) are determined initially, independently with a small probability $q$.
That means that the $failed$ flag of each failed neuron $u$ is set to $1$ at time $0$, $failed^u(0) = 1$, and remains $1$ thereafter, $failed^u(t) = 1$ for every $t$.
We do not here consider failures during learning, but leave that for future work.

We consider a particular concept hierarchy $\mathcal C$, with concept hierarchy notation as defined in Section~\ref{sec: concepts}. 
For our networks, we use notation as defined in Section~\ref{sec: networks}.

Our recognition problem definitions rely on the following definition of how a particular set $B$ of level $0$ concepts is ``presented'' to the network.  This involves firing exactly the input neurons that represent these level $0$ concepts.

\begin{definition}
\label{def:presented}[\textbf{Presented}]
If $B \subseteq D_0$ and $t \geq 0$, then we say that $B$ is \emph{presented at time} $t$ (in some particular network execution) exactly if the following holds.
For every layer $0$ neuron $u$:
\begin{enumerate}
    \item  If $u \in reps(B)$ and $failed^u = 0$, then $u$ fires at time $t$.
    \item  If $u \notin reps(B)$ or $failed^u = 1$, then $u$ does not fire at time $t$.
\end{enumerate}
\end{definition}

\subsection{Recognition in feed-forward networks}

In this subsection and the next, we define what it means for network $\mathcal N$ to recognize concept hierarchy $\mathcal C$.
This section assumes $\mathcal N$ is a feed-forward network and the next section allows $\mathcal N$ to include lateral edges.

In each case, the definition assumes that every concept $c \in C$, at every level $> 0$, has a size-$m$ set of representing neurons, $reps(c)$.

In both of our problem statements, we require that, for each level $\ell$ concept $c$, $0 \leq \ell \leq \ell_{max}$, that is $r_2$-supported by $B$,  with probability at least $1-\delta_{\ell}$, at least $(1 - \epsilon)m$ of the $reps(c)$ neurons should fire.
On the other hand, if $c$ is not $r_1$-supported by $B$, then none of the $reps(c)$ neurons should fire.\footnote{This is an asymmetric definition, because the firing result is probabilistic but the non-firing result is not.
We could make the non-firing result probabilistic, but that would mean relying on the occurrence of failures, which is a bad idea.
}

\begin{definition}[\textbf{Recognition problem for feed-forward networks}]
\label{def: recog-ff}
Network $\mathcal N$ $(r_1,r_2)$-\emph{recognizes} $\mathcal C$
provided that, for each concept $c \in C$ with $level(c) = \ell$, $0 \leq \ell \leq \ell_{max}$, there is a designated set of $m$ neurons, $reps(c)$, such that the following holds.
Assume that $B \subseteq C_0$ is presented at time $0$.  
Then:
\begin{enumerate}
\item
\emph{When $reps(c)$ neurons should fire:}
If $c \in supp_{r_2}(B)$, then with probability at least $1-\delta_{\ell}$, at least $(1 - \epsilon) m$ of the neurons $v \in reps(c)$ neurons fire, each such $v$ at time $layer(v)$.
\item
\emph{When $reps(c)$ neurons should not fire:}
If $c \notin supp_{r_1}(B)$, then no neuron $v \in reps(c)$ fires at time $layer(v)$.
\end{enumerate}
\end{definition}
Naturally, in Condition 1, we would like to minimize the values of $\delta_{\ell}$ and $\epsilon$, in terms of the other parameters.

A word about the two choices here.  
We assume that the random choices are made first, for which neurons fail.
Then the set $B$ is chosen, nondeterministically, potentially by an adversary.\footnote{This allows the set $B$ to be chosen with knowledge of what neurons have failed, which sounds bad.  But later in the paper, "survival lemmas" will tell us that with high probability, every concept has a high probability of having many non-failed $rep$ neurons.}

\subsection{Recognition in networks with lateral edges}

The second definition assumes that $\mathcal N$ is a network with lateral edges.
For this, the timing is harder to pin down, so we formulate the definition a bit differently.
We assume here that the input is presented continuously from some time $0$ onward, and we allow some flexibility in when the $reps(c)$ neurons are required to fire.

\begin{definition}[\textbf{Recognition problem for networks with lateral edges}]
\label{def: recog-lat}
Network $\mathcal N$ $(r_1,r_2)$-\emph{recognizes} $\mathcal C$
provided that, for each concept $c \in C$, with $level(c) = \ell$, $0 \leq \ell \leq \ell_{max}$, there is a designated set of $m$ neurons,
$reps(c)$, such that the following holds.
Assume that $B \subseteq C_0$ is presented at all times $\geq 0$.  
Then:
\begin{enumerate}
\item
\emph{When $reps(c)$ neurons should fire:}
If $c \in supp_{r_2}(B)$, then with probability at least $1-\delta_{\ell}$, at least $(1 - \epsilon) m$ of the neurons in $reps(c)$ fire at all times starting from some time $\geq 0$.
\item
\emph{When $reps(c)$ neurons should not fire:}
If $c \notin supp_{r_1}(B)$, then no neuron in $reps(c)$ fires at any time.
\end{enumerate}
\end{definition}
Again, we assume that the failures are determined first, randomly, and then the set $B$ is chosen, nondeterministically.

\subsection{Learning}

We do not give formal definitions of a "learning problem".
Each solution to one of our recognition problems depends on a particular representation for the concept hierarchy, which we will describe for our three types of networks in Sections~\ref{sec: representation}, \ref{sec: representation-low}, and~\ref{sec: representation-lat-low}.
The job of a learning algorithm is to produce that representation, starting from a default network configuration.
Since we do not know at this point in the paper what the representations should be, we cannot give a general definition of the learning problem here.

\section{Recognition in Feed-Forward Networks with High Connectivity}
\label{sec: recog-ff}

We begin with the simplest case, feed-forward networks with high connectivity, by which we mean total connectivity with weight $1$ edges from $reps$ of children to $reps$ of their parents.\footnote{Technically, we are assuming total connectivity between adjacent layers.  To remove some edges from consideration, here and later in the paper, we simply set their weights to $0$.  That implies that they cannot contribute to their target neuron's incoming potential.}
Total connectivity from $reps$ of children to $reps$ of parents is what we assumed in~\cite{LM21}.
However, that paper considered only the special case of $m = 1$ $reps$ per concept, and did not consider neuron failures.
Now we consider larger values of $m$ and (initial) neuron failures.  

This section can be considered a "warm-up" for Sections~\ref{sec: recog-ff-low} and~\ref{sec: recog-lat-low}.
In those sections, we introduce the complication of partial connectivity, i.e., missing edges, in addition to the partial information and failures considered here.

In the network of this section, we assume that $\ell'_{max} = \ell_{max}$, that is, the number of layers in the network is the same as the number of levels in the concept hierarchy.
Moreover, every concept $c \in C$ has exactly $m$ representing neurons, $reps(c)$, and every neuron $v \in reps(c)$ has $layer(v) = level(c)$.

The main result of this section is Theorem~\ref{thm: main}, which asserts correctness of recognition.
It has two parts, a positive result in Theorem~\ref{thm: main-firing} and a negative result in Theorem~\ref{thm: main-non-firing}, corresponding to the two parts of the recognition definition.
The positive result yields particular values of the recognition approximation parameter $\epsilon$ and the recognition probability parameters $\delta_{\ell}$, $0 \leq \ell \leq \ell_{max}$, in terms of the given parameters $k$, $\ell_{max}$, $m$, and $p$.
In particular, it shows how the values of $\epsilon$ and $\delta_{\ell}$ depend on $m$ and $p$.


\subsection{Parameters}
\label{sec: parameter-values}

We introduce a new parameter:
\begin{itemize}
    \item $\zeta$, a concentration parameter for Chernoff bounds, as in Appendix~\ref{app: prob}.
\end{itemize}
We also assume the following values for parameters introduced earlier:
\begin{itemize}
    \item  $\tau$, the firing threshold for all neurons at layers $\geq 1$, is defined to be $r_2 k m p (1 - \zeta)$.
    \item $\epsilon$, the recognition approximation parameter, is defined to be $1 - p(1 - \zeta)$, where $p = 1 - q$ is the survival probability for each individual neuron.
    \item For every $\ell$, $0 \leq \ell \leq \ell_{max}$,  the recognition probability parameter $\delta_{\ell}$ is defined to be 
    \[\frac{k^{\ell+1} - 1}{k-1} \cdot \exp\left(-\frac{m p \zeta^2}{2}\right).\]
    In particular, $\delta_{0} = \exp\left(-\frac{m p \zeta^2}{2}\right)$.
\end{itemize}

\subsection{The representation of concept hierarchy $\mathcal C$}
\label{sec: representation}.

We define a feed-forward network $\mathcal N$ that is specially tailored to recognize concept hierarchy $\mathcal C$.  This is for a representation that has already been learned; we will discuss learning for this case in Section~\ref{sec: learning}.  
The learning approach will be a variant of what we used in the noise-free algorithm in~\cite{LM21}.

The strategy is simply to embed the digraph induced by $\mathcal C$ in the network $\mathcal N$, using a redundant representation.
For every level $\ell$ concept $c$ of $\mathcal C$, we assume a set $reps(c)$ of $m$ designated neurons in layer $\ell$ of the network; all such sets are disjoint.
Let $R$ be the set of all representatives of concepts, that is, $R = \bigcup_c reps(c)$.
Let $rep^{-1}$ denote the corresponding inverse function that gives, for every $u \in R$, the concept $c \in C$ with $u \in reps(c)$.

We define the weights of the edges as follows.
If $v$ is any layer $\ell$ neuron,  $1 \leq \ell \leq \lmax$, and $u$ is any layer $\ell-1$ neuron, then we define the edge weight $weight(u,v)$ to be:\footnote{Note that these simple weights of $1$ and $0$ do not correspond precisely to what is achieved by the noise-free learning algorithm in~\cite{LM21}.
There, the algorithm approaches the following weights, in the limit:
\[ weight(u,v) = 
\begin{cases}
\frac{1}{\sqrt{k}} & \text{ if } u,v \in R \text{ and } rep^{-1}(u) \in children(rep^{-1}(v)),
\\ 
0 & \text{ otherwise.}
\end{cases} \]
We proved in~\cite{LM21} that, after a certain number of steps of the noise-free learning algorithm, the weights are sufficiently close to these limits to guarantee that network $\mathcal N$ $(r_1,r_2)$-recognizes $\mathcal C$.  In~\cite{DBLP:conf/sirocco/LynchM23}, we showed formally that results for the simple case of weights in $\{0,1\}$ carry over to the approximate, scaled version that is actually produced by the learning algorithm of~\cite{LM21}.}
\[ weight(u,v) = 
\begin{cases}
1 & \text{ if } u,v \in R \text{ and } rep^{-1}(u) \in children(rep^{-1}(v)),
\\ 
0 & \text{ otherwise.}
\end{cases}
\]
Thus, we have total connectivity with weight $1$ edges from $reps$ of children to $reps$ of their parents.  All other edges from layer $\ell-1$ to layer $\ell$ have weight $0$.\footnote{We can think of these edges as "missing" but as noted earlier, they are present, but with weight $0$.}

As noted in Section~\ref{sec: parameter-values}, we define the threshold $\tau$ for any non-input neuron $u$ to be $r_2 k m p (1 - \zeta)$.
The $p$ and $(1 - \zeta)$ factors are included in order to accommodate the possible failure of some of the neurons at layer $layer(u) - 1$.


Finally, we assume that the initial $firing$ and $engaged$ values for all the non-input neurons are $0$.
This completely defines network $\mathcal N$ and, together with the nondeterministic choice of input set $B$ and random choice of failed neurons, determines its behavior.



\subsection{The main theorem}

We consider a particular concept hierarchy $\mathcal C$, with concept hierarchy notation as defined in Section~\ref{sec: concepts}.
We define the parameters as in Section~\ref{sec: parameter-values} and 
define the network $\mathcal N$ as in Section~\ref{sec: representation}.
In terms of all of these, our main result is:

\begin{theorem}
\label{thm: main}
Assume that $r_1 \leq r_2 p (1 - \zeta)$ and $r_2 > 0$.
Then $\mathcal N (r_1,r_2)$-recognizes $\mathcal C$.
\end{theorem}

We prove Theorem~\ref{thm: main} in two parts, following the definition of $(r_1,r_2)$-recognition.
Theorem~\ref{thm: main-firing} expresses the firing guarantee, and Theorem~\ref{thm: main-non-firing} expresses the non-firing guarantee.  Theorems~\ref{thm: main-firing} and~\ref{thm: main-non-firing} together immediately imply Theorem~\ref{thm: main}.

\subsection{Survival lemmas}
\label{sec: prob-lemmas1}

The next section, Section~\ref{sec: guar-firing}, is devoted to proving a probabilistic guarantee for firing, in Theorem~\ref{thm: main-firing}.  In this section, we give two preliminary lemmas.

According to our definitions, for the simple case of feed-forward networks with high connectivity, the only probabilistic behavior of the network arises from the independent initial failures of the individual neurons in the network.  In this section, we give two simple lemmas that provide bounds on the number of surviving (i.e., non-failing) neurons in the $reps$ sets.
These serve to encapsulate all of the probabilistic reasoning that is needed to prove Theorem~\ref{thm: main-firing}.

The first lemma bounds the probability of "sufficient survival" for the $reps$ of a particular concept $c$.
Recall that we are assuming that each neuron fails, independently, with probability $q = 1-p$.

\begin{lemma}
\label{lem: reps-for-one-concept}
For every concept $c$ in the concept hierarchy $\mathcal C$, at any level, the probability that the number of surviving neurons in $reps(c)$ is $\leq m p(1 - \zeta)$ is at most $\delta_0 = \exp\left(-\frac{m p \zeta^2}{2}\right)$.   
\end{lemma}

\begin{proof}
We use a Chernoff bound.  
The probability of failure for each neuron in $reps(c)$ is $q$.  
To upper-bound the probability that no more than $m p(1-\zeta)$ of the $m$ $reps$ survive, we use a lower tail bound, in the form given in Appendix~\ref{app: prob}:

\[
\text{For any } \zeta \in [0,1], \Pr[X \leq (1-\zeta) \mu] \leq \exp\left(-\frac{\mu \zeta^2}{2}\right).
\]
\noindent
In our case, the mean $\mu$ is equal to $m p$.
So this says that $Pr[X \leq m p (1 - \zeta)] \leq \exp\left(-\frac{m p \zeta^2}{2}\right)$.
\end{proof}

We extend Lemma~\ref{lem: reps-for-one-concept} to give an overall probability of "sufficient survival" for all descendants of $c$:

\begin{lemma}
\label{lem:  reps-for-all-descendants}
Consider a particular concept $c$ with $level(c) = \ell$, $0 \leq \ell \leq \ell_{max}$.
Let $A$ be the event that, for some descendant $c'$ of $c$ (possibly $c$ itself), the number of surviving neurons in $reps(c')$ is $\leq m p (1 - \zeta)$. 
Then $Pr(A) \leq \delta_{\ell} = \frac{k^{\ell+1} - 1}{k-1} \cdot \exp\left(-\frac{m p \zeta^2}{2}\right)$.
\end{lemma}

\begin{proof}
The number of descendants of $c$ is exactly $1 + k + k^2 + \cdots + k^{\ell} = \frac{k^{\ell+1} - 1}{k-1}$.  A union bound, taken over all of these descendants, and Lemma~\ref{lem: reps-for-one-concept} yield the result. 
\end{proof}

To give an idea of the size of $Pr(A)$ in Lemma~\ref{lem:  reps-for-all-descendants}, we imagine that the parameters $k$ and $\ell$ of the concept hierarchy are small, say $4$.
The number $m$ of $reps$ of a concept should be fairly large, say $320$.
The survival probability $p$ for an individual neuron should be close to $1$, say $\frac{31}{32}$.
The concentration parameter $\zeta$ might be around $1/4$.
With these values, the probability in this last result is approximately
$4^4 \cdot \exp\left(-\frac{(320)(31/32) (1/4)^2}{2}\right)$, or $256 \cdot \exp\left(-9.7\right)$.
This last expression evaluates to approximately $256 / 16000$, or $.016$.

\subsection{Proof of guaranteed firing}
\label{sec: guar-firing}

Now we prove our main firing theorem.

\begin{theorem}
\label{thm: main-firing}
Assume that a set $B \subseteq C_0$ is presented at time $0$.
Let $c$ be a level $\ell$ concept, $0 \leq \ell \leq \ell_{max}$, that is in $supp_{r_2}(B)$.
Then with probability at least $1 - \delta_{\ell}$, at least $m p(1 - \zeta)$ of the neurons in $reps(c)$ fire at time $\ell$.
\end{theorem}

\begin{proof}
Fix $B \subseteq C_0$, and fix a particular concept $c$ at level $\ell$ that is in $supp_{r_2}(B)$.
Consider the event $\overline{A}$, where $A$ is the event defined in the statement of Lemma~\ref{lem:  reps-for-all-descendants}.
The event $\overline{A}$ says that every descendant $c'$ of $c$ (including $c$ itself) has more than $mp(1 - \zeta)$ surviving neurons in $reps(c')$.
By Lemma~\ref{lem:  reps-for-all-descendants}, we have that $Pr(A) \leq \delta_{\ell}$.

For the rest of the proof, we condition on event $\overline{A}$, that is, we assume that every descendant $c'$ of $c$ has more than $mp(1-\zeta)$ surviving neurons in $reps(c')$.  
We prove that, under this assumption, at least $m p(1-\zeta)$ of the neurons in $reps(c)$ fire. This follows from the following Claim:

\begin{claim}  
\label{claim: firing}
If $c'$ is a descendant of $c$ with $level(c') = \ell'$, $0 \leq \ell' \leq \ell$, with $c' \in supp_{r_2}(B)$, then at least $m p (1 - \zeta)$ of the neurons in $reps(c')$ fire at time $\ell'$. 
\end{claim}


\noindent
\emph{Proof of Claim~\ref{claim: firing}:}
We prove this by induction on $\ell'$.

\emph{Base:} $\ell' = 0$: 
Consider any level $0$ descendant $c'$ of $c$ that is supported by $B$.  Then $c' \in B$.
Then since we are conditioning on the event $\overline{A}$, more than $mp(1 - \zeta)$ of the neurons in $reps(c')$ survive, and so they fire at time $0$.
This suffices for the base case.

\emph{Inductive step:} $1 \leq \ell' \leq \ell$:
Suppose that the claim holds for $\ell'-1$ and consider a level $\ell'$ descendant $c'$ of $c$ that is supported by $B$.  Then by the definition of "supported", at least $r_2 k$ children of $c'$ at level $\ell'-1$ are supported by $B$.  By the inductive hypothesis, each of these children has at least $m p (1 - \zeta)$ of its $reps$ firing at time $\ell'-1$.  So there are at least $r_2 k m p (1 - \zeta)$ level $\ell'-1$ neurons that are $reps$ of children of $c'$ and fire at time $\ell'-1$.  

By our connectivity assumption, all of these firing $reps$ are connected to each neuron $v$ in $reps(c')$ with weight $1$ edges, so the firing threshold of each such $v$, which is $r_2 k m p (1 - \zeta)$, is met.  
So if $v$ does not fail, it will fire at time $\ell'$; thus, in the absence of failures, all $m$ of the neurons in $reps(c')$ would fire at time $\ell'$.  
However, we do need to take account of failures.

Because we are conditioning on event $\overline{A}$, we know that more than $m p (1 - \zeta)$ of the neurons in $reps(c')$ survive.  Since all of their firing thresholds are met, all of these surviving neurons fire at time $\ell'$.
That means that more than $m p (1 - \zeta)$ of the neurons in $reps(c')$ fire at time $\ell'$, which suffices for the inductive step. \\
\emph{End of proof of Claim~\ref{claim: firing}.}

Instantiating Claim~\ref{claim: firing} with $c' = c$ yields that, conditioned on $\overline{A}$, at least $m p (1 - \zeta)$ of the neurons in $reps(c)$ fire at time $\ell$. 
Since $Pr(\overline{A}) \geq 1 - \delta_{\ell}$,
the Theorem follows.
\end{proof}

An interesting aspect of our proof is that all of the probabilistic arguments are encapsulated within the survival lemmas in Section~\ref{sec: prob-lemmas1}; these assert certain levels of survival with high probability.  The arguments in this section are entirely non-probabilistic, showing firing guarantees assuming certain minimum levels of survival.

\subsection{Proof of guaranteed non-firing}
\label{sec: guar-non-firing}

And now we prove our main non-firing theorem.

\begin{theorem}
\label{thm: main-non-firing}
Assume that $r_2 > 0$ and $r_1 \leq r_2 p (1 - \zeta)$.
Assume that a set $B \subseteq C_0$ is presented at time $0$.
Let $c$ be a level $\ell$ concept, $0 \leq \ell \leq \ell_{max}$, that is not in $supp_{r_1}(B)$.
Then none of the neurons in $reps(c)$ fire at time $\ell$.
\end{theorem}

Notice that, unlike Theorem~\ref{thm: main-firing}, Theorem~\ref{thm: main-non-firing}
does not mention probabilities or a fraction of the $reps$; this is because we are ignoring failures for the lower bound.  
If we were to take account of failures here, it would mean that we are counting on a certain number of failures to occur (with high probability), which seems like a bad idea.
This is because failure estimates are generally designed to be conservative, that is, to be upper bounds on the number of failures, rather than precise estimates.

\begin{proof}
(Sketch:)  
The argument is similar to that of Lemma 5.2 in our previous paper~\cite{DBLP:conf/sirocco/LynchM23}.
The main change arises from the lower threshold, $r_2 k m p (1 - \zeta)$.   This requires a change in the assumption about the separation between $r_1$ and $r_2$.
Namely, we now require $r_1 \leq r_2 p (1 - \zeta)$ instead of just $r_1 \leq r_2$ as in~\cite{DBLP:conf/sirocco/LynchM23}.
\end{proof}



\section{Recognition in Feed-Forward Networks with Low Connectivity}
\label{sec: recog-ff-low}

Our connectivity assumptions in Section~\ref{sec: recog-ff} are probably stronger than what occurs in real brains.  
In this section, we weaken the connectivity assumptions, at the cost of somewhat weaker firing guarantees. 
Formally, this weakening means that some of the edges from $reps$ of children to $reps$ of parents may have weight $0$.

Specifically, we reduce the number of assumed weight $1$ edges from $reps$ of children of a concept $c$ to a $rep$ of $c$ from $k m$ to $a k m$, for some constant $a$ between $0$ and $1$.  We imagine that $a$ should be fairly large, say approximately $3/4$.

However, we must be careful about which edges we exclude.  The danger is that the choice of the presented set $B$ might interact badly with the choice of excluded edges; for example, the excluded edges might all happen to involve $reps$ of concepts supported by $B$.
We must formulate the connectivity assumption in a way that takes account of such interactions.

Our approach is to refine the connectivity assumptions so that they are expressed in terms of the individual children of a given concept $c$, rather than in terms of all the children of $c$ in the aggregate. 
Specifically, for every $rep$ $v$ of $c$, and every child $c'$ of $c$, $v$ will have a large fraction of connections from $reps$ of $c'$.
The resulting finer-granularity connectivity assumptions allow us to ignore the dependencies between the choices of $B$ and of the missing edges, and obtain reasonable firing guarantees.




In the network of this section, we again assume that $\ell'_{max} = \ell_{max}$, that is, the number of layers in the network is the same as the number of levels in the concept hierarchy.
Moreover, every concept $c \in C$ has exactly $m$ representing neurons, $reps(c)$, and every neuron $v \in reps(c)$ has $layer(v) = level(c)$.

The rest of this section follows the general outline of Section~\ref{sec: recog-ff}.

\subsection{Parameters}
\label{sec: parameter-values-low}

We introduce the following new parameter:
\begin{itemize}
\item $a \in [0,1]$, a coefficient to describe a lower bound on connectivity.
\end{itemize}

We also assume the following values for parameters introduced earlier.  Compared to Section~\ref{sec: recog-ff}, we are lowering the firing threshold $\tau$ and modifying the $\delta_{\ell}$ probabilities.
\begin{itemize}
    \item  $\tau$, the firing threshold for all neurons at layers $\geq 1$, is defined to be $a r_2 k m p (1 - \zeta)$.
    \item $\epsilon$, the recognition approximation parameter, is defined to be $1 - p(1 - \zeta)$. 
    \item $\delta_{0} = \exp\left(-\frac{m p \zeta^2}{2}\right)$, and for $\ell \geq 1$, 
    \[ \delta_{\ell} = \frac{k^{\ell+1} - 1}{k-1} \cdot \exp\left(-\frac{m p \zeta^2}{2}\right) + \frac{k^{\ell} - 1}{k-1} \cdot k m \cdot \exp\left(-\frac{am p \zeta^2}{2}\right).\]
\end{itemize}

\subsection{The representation of concept hierarchy $\mathcal C$}
\label{sec: representation-low}

As in Section~\ref{sec: recog-ff}, we define a feed-forward network $\mathcal N$ that is specially tailored to recognize concept hierarchy $\mathcal C$. 
As before, we embed the digraph induced by $\mathcal C$ in the network $\mathcal N$, using a redundant representation.
Thus, for every level $\ell$ concept $c$ of $\mathcal C$, we assume a set $reps(c)$ of $m$ designated neurons in layer $\ell$ of the network; all such sets are disjoint.
Let $R = \bigcup_c reps(c)$.
Let $rep^{-1}$ denote the corresponding inverse function that gives, for every $u \in R$, the concept $c \in C$ with $u \in reps(c)$.

Now we assume that, for any level $\ell$ concept $c$ in the concept hierarchy, $1 \leq \ell \leq \ell_{max}$, for every neuron $v$ in $reps(c)$, and for every child $c'$ of $c$, neuron $v$ has at least $a m$ incoming weight $1$ edges from $reps$ of $c'$. 
More precisely, for each child $c'$ of $c$ and $v \in reps(c)$, define $inc(v,c')$ to be the set of all neurons in $reps(c')$ that have weight $1$ edges connecting them to $v$.
Then our connectivity assumption says that $|inc(v,c')| \geq a m$.
All other edges from layer $\ell-1$ to layer $\ell$ have weight $0$.

As noted in Section~\ref{sec: parameter-values-low}, we define the threshold $\tau$ for any non-input neuron to be $a r_2 k m p (1 - \zeta)$.
We assume that the initial $firing$ and $engaged$ values for all the non-input neurons are $0$.


\subsection{The main theorem}

We consider a particular concept hierarchy $\mathcal C$, with concept hierarchy notation as defined in Section~\ref{sec: concepts}.
We define the parameters as in Section~\ref{sec: parameter-values-low} and 
define the network $\mathcal N$ as in Section~\ref{sec: representation-low}.
Our main result is:

\begin{theorem}
\label{thm: main-low}
Assume that $r_1 \leq a r_2 p (1 - \zeta)$ and $r_2 > 0$.
Then $\mathcal N (r_1,r_2)$-recognizes $\mathcal C$.
\end{theorem}

We prove Theorem~\ref{thm: main-low} in two parts.
Theorem~\ref{thm: main-firing-low} expresses the firing guarantee, and Theorem~\ref{thm: main-non-firing-low} expresses the non-firing guarantee.  Theorems~\ref{thm: main-firing-low} and~\ref{thm: main-non-firing-low} together immediately imply Theorem~\ref{thm: main-low}.

\subsection{Survival lemmas}

We begin with two lemmas analogous to Lemmas~\ref{lem: reps-for-one-concept} and~\ref{lem:  reps-for-all-descendants}.  Proofs are as before.

\begin{lemma}
\label{lem: reps-for-one-concept-low}
For every concept $c$ in the concept hierarchy $\mathcal C$, at any level, the probability that the number of surviving neurons in $reps(c)$ is $\leq m p (1 - \zeta)$ is at most $\delta_0 = \exp\left(-\frac{m p \zeta^2}{2}\right)$. 
\end{lemma}

\begin{lemma}
\label{lem:  reps-for-all-descendants-low}
Consider a particular concept $c$ with $level(c) = \ell$, $0 \leq \ell \leq \ell_{max}$.
Let $A$ be the event that, for some descendant $c'$ of $c$ (possibly $c$ itself), the number of surviving neurons in $reps(c')$ is $\leq m p (1 - \zeta)$. 
Then $Pr(A) \leq \frac{k^{\ell+1} - 1}{k-1} \cdot \exp\left(-\frac{m p \zeta^2}{2}\right)$.
\end{lemma}

Now we have some new lemmas, which consider separately the number of survivals among the $reps$ of each child of a concept.
Namely, we consider a single neuron $v$ in $reps(c)$ and a single child $c'$ of $c$, and give a lower bound for the number of surviving neurons in $inc(v,c')$, which is the set of neurons in $reps(c')$ that are connected to $v$ by weight $1$ edges.  Proofs again use Chernoff bounds and union bounds.

\begin{lemma}
\label{lem: one-concept-incoming-low}
Consider a particular concept $c$ with $level(c) = \ell \geq 1$.
Let $v \in reps(c)$.
Then for each child $c'$ of $c$, the probability that the number of surviving neurons in $inc(v,c')$ is $\leq amp (1-\zeta)$ is at most $\exp\left(-\frac{am p \zeta^2}{2}\right)$.
\end{lemma}

\begin{proof}
By Chernoff, using a mean of $amp$ and concentration parameter of $\zeta$.
Since the size of $inc(v,c')$ may be larger than $am$, we can consider any subset of $inc(v,c')$ of size exactly $am$ and use a monotonicity argument.
\end{proof}

\begin{lemma}
\label{lem: all-descendants-incoming-low}
Consider a particular concept $c$ with $level(c) = \ell \geq 1$.
Let $A'$ be the event that, for some descendant $c'$ of $c$ (possibly $c$ itself) with $level(c') \geq 1$, for some $v \in reps(c')$, and for some child $c''$ of $c'$, the number of surviving neurons in $inc(v,c'')$ is $\leq a m p (1 - \zeta)$. 
Then $Pr(A') \leq \frac{k^{\ell} - 1}{k-1} \cdot k m \cdot \exp\left(-\frac{am p \zeta^2}{2}\right)$.
\end{lemma}

\begin{proof}
The number of descendants $c'$ of $c$ with level $\geq 1$ is $\frac{k^{\ell} - 1}{k-1}$.  The number of $reps$ for each such $c'$ is $m$.  The number of children of each such $c'$ is $k$.
A union bound, taken over all of these descendants, $reps$, and children, and Lemma~\ref{lem: one-concept-incoming-low}, yield the result. 
\end{proof}

We finish with a lemma that combines the results of Lemmas~\ref{lem: reps-for-all-descendants-low} and~\ref{lem: all-descendants-incoming-low}, for concepts $c$ with $level(c) \geq 1$.  The proof uses another union bound

\begin{lemma}
\label{lem: union-of-events-low}
Consider a particular concept $c$ with $level(c) = \ell \geq 1$.
Define $A$ as in the statement of Lemma~\ref{lem:  reps-for-all-descendants-low} and
$A'$ as in the statement of Lemma~\ref{lem: all-descendants-incoming-low}.
Then \[Pr(A \cup A') \leq \delta_{\ell} = 
\frac{k^{\ell+1} - 1}{k-1} \cdot \exp\left(-\frac{m p \zeta^2}{2}\right)  + \frac{k^{\ell} - 1}{k-1} \cdot km \cdot \exp\left(-\frac{am p \zeta^2}{2}\right).\]
\end{lemma}

To give an idea of the size of $Pr(A \cup A')$ in Lemma~\ref{lem:  union-of-events-low}, we assume, as before, that the values of $k$ and $\ell$ are both $4$,
the survival probability $p$ for an individual neuron is $\frac{31}{32}$, and
the concentration parameter $\zeta$ is $1/4$.
Now we assume that the number $m$ of $reps$ of a concept is larger than before, say $640$.
We assume that the value of $a$ is $3/4$.
With these values, the probability in this last result is bounded by the sum of two terms.
The first is approximately
$4^4 \cdot \exp\left(-\frac{(640)(31/32) (1/4)^2}{2}\right)$, or $256 \cdot \exp\left(-19.375\right)$.
The second is approximately
$4^3 \cdot 4 \cdot 640 \cdot \exp\left(-\frac{(480)(31/32) (1/4)^2}{2}\right)$, or $256 \cdot 640 \cdot \exp\left(-14.5\right)$.
The first of these two terms is negligible; the second is approximately $.083$.

\subsection{Proof of guaranteed firing}
\label{sec: guar-firing-low}

Our main firing theorem, Theorem~\ref{thm: main-firing-low}, is analogous to Theorem~\ref{thm: main-firing}.  
Notice that Theorem~\ref{thm: main-firing-low} still talks about firing of at least $m p (1- \zeta)$ of the $m$ neurons in the set $reps(c)$, which is the same fraction as in Theorem~\ref{thm: main-firing}.  
However, the probability in the statement of the theorem changes, since it uses the new definition of $\delta_{\ell}$.

\begin{theorem}
\label{thm: main-firing-low}
Assume that a set $B \subseteq C_0$ is presented at time $0$.
Let $c$ be a level $\ell$ concept, $0 \leq \ell \leq \ell_{max}$, that is in $supp_{r_2}(B)$.  Then with probability at least $1 - \delta_{\ell}$, at least $m p (1 - \zeta)$ of the neurons in $reps(c)$ fire at time $\ell$.
\end{theorem}

For the proof, we find it convenient to separate the cases of $\ell = 0$ and $\ell \geq 1$, in the following two lemmas.

\begin{lemma}
\label{lem: 0-case-low}
Assume that a set $B \subseteq C_0$ is presented at time $0$.
Let $c$ be a level $0$ concept that is in $supp_{r_2}(B)$.
Then with probability at least $1 - \delta_{0}$, at least $m p (1 - \zeta)$ of the neurons in $reps(c)$ fire at time $0$.
\end{lemma}

\begin{proof}
Fix $B \subseteq C_0$, and fix a particular level $0$ concept $c$ that is in $supp_{r_2}(B)$.
Since $c \in supp_{r_2}(B)$ and $level(c) = 0$, we have that $c \in B$.
By the definition of presenting $B$ at time $0$, all of the surviving neurons in $reps(c)$ fire at time $0$.
By Lemma~\ref{lem: reps-for-one-concept-low}, 
with probability at least $1 - \delta_0$, the number of surviving neurons in $reps(c)$ is at least $mp (1 - \zeta)$.
Therefore, with probability at least $1 - \delta_0$, at least $m p (1 - \zeta)$ of the neurons in $reps(c)$ fire at time $0$, as needed.
\end{proof}

\begin{lemma}
\label{lem: nonzero-case-low}
Assume that a set $B \subseteq C_0$ is presented at time $0$.
Let $c$ be a level $\ell$ concept, $1 \leq \ell \leq \ell_{max}$, that is in $supp_{r_2}(B)$.
Then with probability at least
$1 - \delta_{\ell}$,
at least $m p (1 - \zeta)$ of the neurons in $reps(c)$ fire at time $\ell$.
\end{lemma}

\begin{proof}
Fix $B \subseteq C_0$, and fix a particular level $\ell$ concept $c$, $1 \leq \ell \leq \ell_{max}$, that is in $supp_{r_2}(B)$.
Define event $A$ as in Lemma~\ref{lem:  reps-for-all-descendants-low} and event $A'$ as in Lemma~\ref{lem: all-descendants-incoming-low}. 
Then by Lemma~\ref{lem: union-of-events-low}, 
$Pr(A \cup A') \leq \delta_{\ell}$.

For the rest of the proof, we condition on the event $\overline{A \cup A'}$.  
That is, we assume that every descendant $c'$ of $c$ has more than $m p (1 - \zeta)$ surviving neurons in $reps(c')$, and also that for every descendant $c'$ of $c$ with $level(c') \geq 1$, every $v \in reps(c')$, and every child $c''$ of $c'$, the number of surviving neurons in $inc(v,c'')$ is greater than $a m p (1 - \zeta)$.

We use the following Claim, which talks both about neurons having their firing thresholds met and about neurons actually firing.

\begin{claim}
\label{claim: firing-low}
If $c'$ is a descendant of $c$ with $level(c') = \ell'$, $ 0 \leq \ell' \leq \ell$, with $c' \in supp_{r_2}(B)$, then:
\begin{enumerate}
    \item  If $\ell' \geq 1$ then every neuron in $reps(c')$ has its threshold met for time $\ell'$.\footnote{Recall that the incoming potential for time $\ell'$ is computed from the firing pattern of incoming neighbors at time $\ell'-1$.}
    \item  At least $m p (1 - \zeta)$ of the neurons in $reps(c')$ fire at time $\ell'$.  
\end{enumerate}
\end{claim}

\noindent
\emph{Proof of Claim~\ref{claim: firing-low}:} 
By induction on $\ell'$, using two base cases. \\

\emph{Base}, $\ell' = 0$:
Part 1 is vacuous.
For Part 2, since $c' \in supp_{r_2}(B)$ and $level(c') = 0$, we have that $c' \in B$.
By the definition of presenting $B$, all of the surviving neurons in $reps(c')$ fire at time $0$.  Since we are conditioning on $\overline{A \cup A'} \subseteq \overline{A}$, the number of surviving neurons in $reps(c')$ is at least $mp(1-\zeta)$.
Therefore at least $mp(1-\zeta)$ of the neurons in $reps(c')$ fire at time $0$.

\emph{Base}, $\ell' = 1$:
To show Part 1, we fix any $v \in reps(c')$ and show that $v$'s threshold is met for time $1$.
Because $c'$ is supported by $B$, $c'$ has at least $r_2 k$ children in $B$.
Consider each supported child $c''$ of $c'$ individually.
Since we are conditioning on $\overline{A \cup A'} \subseteq \overline{A'}$, we have that at least $a m p(1-\zeta)$ neurons in $inc(v,c'')$ survive.
Since $level(c'') = 0$, the definition of presenting $B$ at time $0$ implies that all of these surviving neurons fire at time $0$.
Taking into account all of the supported children $c''$ of $c'$, we have 
at least $a r_2 k m p(1 - \zeta)$ weight $1$ connections to $v$ from $reps$ of children of $c'$ that fire at time $0$.
This meets $v$'s firing threshold for time $1$, showing Part 1.

For Part 2, because we are conditioning on $\overline{A \cup A'} \subseteq \overline{A}$, we know that more than $m p (1 - \zeta)$ of the neurons in $reps(c')$ survive.  
Since, by Part 1, all of their firing thresholds are met for time $1$, we have that more than $mp(1 - \zeta)$ of the neurons in $reps(c')$ fire at time $1$. 

\emph{Inductive step}, $2 \leq \ell' \leq \ell$:  
We show Part 1; then, as in the previous case, Part 2 follows since we are conditioning on $\overline{A \cup A'} \subseteq \overline{A}$.

To show Part 1, we fix any $v \in reps(c')$ and show that $v$'s threshold is met for time $\ell'$.
Because $c'$ is supported by $B$, $c'$ has at least $r_2 k$ children that are supported by $B$.
Consider each supported child $c''$ of $c'$ individually.
Since we are conditioning on $\overline{A \cup A'} \subseteq \overline{A'}$, we have that at least $a m p(1-\zeta)$ neurons in $inc(v,c'')$ survive.
Since $level(c'') \geq 1$, by the inductive hypothesis, Part 1, these all have their thresholds met for time $\ell'-1$, and therefore they fire at time $\ell'-1$.
Taking into account all of the supported children $c''$ of $c'$, we have 
at least $a r_2 k m p(1 - \zeta)$ weight $1$ connections to $v$ from $reps$ of children of $c'$ that fire at time $\ell' - 1$.
This meets $v$'s firing threshold for time $\ell'$, showing Part 1. \\
\emph{End of proof of Claim~\ref{claim: firing-low}.}

Instantiating Part 2 of Claim~\ref{claim: firing-low} with $c' = c$ yields that, conditioned on $\overline{A \cup A'}$, at least $m p (1 - \zeta)$ of the neurons in $reps(c)$ fire at time $\ell$.
Since $Pr(A \cup A') \leq \delta_{\ell}$, the Lemma follows.
\end{proof} 

\begin{proof} (Of Theorem~\ref{thm: main-firing-low}):
Follows from Lemmas~\ref{lem: 0-case-low} and~\ref{lem: nonzero-case-low}. 
\end{proof}

\paragraph{Note:}  Our use of a fine-granularity connectivity assumption in this section has impact on the size of the $\delta_{\ell}$ probabilities for a given number $m$ of $reps$.  It should be possible to reduce these probabilities 
by making additional randomness assumptions, for example, by assuming that the presented set $B$ is chosen randomly from some distribution.  We avoid considering this, for now.

\subsection{Proof of guaranteed non-firing}

And now we prove our main non-firing theorem.

\begin{theorem}
\label{thm: main-non-firing-low}
Assume that $r_2 > 0$ and $r_1 \leq a r_2 p(1 - \zeta)$.
Assume that a set $B \subseteq C_0$ is presented at time $0$.  
Let $c$ be a level $\ell$ concept, $0 \leq \ell \leq \ell_{max}$, that is not in $supp_{r_1}(B)$.
Then none of the neurons in $reps(c)$ fire at time $\ell$.
\end{theorem}

\begin{proof} (Sketch:)
As in Section~\ref{sec: guar-non-firing}, the argument considers the most favorable case for an algorithm.
Here that means no failures and all-to-all connectivity with weight $1$ edges from $reps$ of children to $reps$ of parents. 

The only change from Section~\ref{sec: guar-non-firing} is a greater separation between $r_1$ and $r_2$, which is needed because of the lower threshold of $a r_2 k m p (1-\zeta)$ used in this section.
The proof again follows arguments like those in~\cite{DBLP:conf/sirocco/LynchM23}.
\end{proof}

\section{Recognition in Networks with Lateral Edges, with Low Connectivity}
\label{sec: recog-lat-low}

In this section, we consider recognition in networks that contain some weight $1$ lateral edges between $reps$ of the same concept, in addition to weight $1$ forward edges from $reps$ of children to $reps$ of their parents.

In Section~\ref{sec: recog-ff-low}, we assumed partial connectivity from $reps$ of children to $reps$ of their parents, expressed in terms of a coefficient $a$.  Namely, for each $rep$ $v$ of a concept $c$, and for each child $c'$ of $c$, we assumed weight $1$ edges from at least $a m$ $reps$ of $c'$ to $v$.
In this section we weaken this assumption, by dividing the $reps$ of concept $c$ into two \emph{Classes}, which we (unimaginatively) call \emph{Class 1} and \emph{Class 2}.  Each $rep$ $v$ in Class $1$ satisfies the same assumption as the $reps$ in Section~\ref{sec: recog-ff-low}:  for each child $c'$ of $c$, we assume weight $1$ edges from at least $a m$ $reps$ of $c'$ to $v$.
For each $rep$ $v$ in Class 2, we weaken this requirement using a smaller threshold $a_1$, namely, for each child $c'$ of $c$, we assume weight $1$ edges from at least $a_1 m$ $reps$ of $c'$ to $v$.
We compensate for the missing weight $1$ forward edges with some weight $1$ lateral edges from other $reps$ of $v$, i.e., from \emph{peer} $reps$.

As in Section~\ref{sec: recog-ff-low}, we must be careful about dependencies between the choice of missing edges and the choice of presented set $B$.  We again deal with these complications by breaking down the connectivity requirements in terms of smaller groups of neurons.  This time, in addition to groups consisting of $reps$ of individual child neurons, we use groups consisting of peer $reps$.

This section is generally inspired by the assembly calculus work in~\cite{DBLP:conf/innovations/Legenstein0PV18,PapadimitriouVempala}.  Our layers are intended to correspond to the areas in~\cite{DBLP:conf/innovations/Legenstein0PV18,PapadimitriouVempala}.
As in our work, those papers deal with representation of concepts with multiple neurons, where neurons that are $reps$ of a concept $c$ are triggered to fire based on a combination of potential from representatives of children of $c$ and from other representatives of the same concept $c$.
However, the papers~\cite{DBLP:conf/innovations/Legenstein0PV18,PapadimitriouVempala} emphasize how representations of structured concepts are learned, whereas we emphasize the form of representations and how they are used for recognition in the presence of partial information and failures.
We discuss learning informally in Sections~\ref{sec: learning} and~\ref{sec: learning-lateral}.

To separate the treatment of recognition from that of learning, we define a new connectivity property that might be guaranteed by a learning algorithm with high probability, and that suffices for reliable recognition.  
We call the new connectivity property the Class Assumption, and define it in Section~\ref{sec: representation-lat-low}.
It captures the division of $reps$ into Class 1 and Class 2, with appropriate connectivity properties for each Class.

We use the Class Assumption in Section~\ref{sec: guar-lat-low} to prove that recognition works correctly.  
In particular, we show that, with the Class Assumption, we have high probability of correctly recognizing a concept that is $r_2$-supported by a presented set $B$.  
As before, all the probability here arises from the randomly-chosen initial neuron failures.
As before, we prove this high-probability claim using "survival lemmas", which we present in Section~\ref{sec: prob-lemmas-lat}.

In the network of this section, we again assume that $\ell'_{max} = \ell_{max}$.
Moreover, every concept $c \in C$ has exactly $m$ representing neurons, $reps(c)$, and every neuron $v \in reps(c)$ has $layer(v) = level(c)$.

\subsection{Parameters}
\label{sec: parameter-values-lat-low}

We introduce:
\begin{itemize}
\item $a_1\in [0,a]$, describing a lower bound that is smaller than $a$, on connectivity between layers.
\item $a_2 \in [0,1]$, describing a lower bound on connectivity within layers; we require that $a_2 \geq (a - a_1) k$.
\item $m_1 \in [0,m]$, describing the number of Class 1 $reps$ of any concept; write $m_2$ for $m - m_1$, which is the number of Class 2 $reps$ of any concept.\footnote{This uniformity is certainly an oversimplification, but then again, we are oversimplifying in many places in this paper, in order to make analysis tractable.}
\end{itemize}

We assume the following values for parameters introduced earlier.
\begin{itemize}
    \item  The threshold $\tau$ for all neurons at layers $\geq 1$ is defined to be $a r_2 k m p(1-\zeta)$.
    \item $\epsilon = 1 - p(1-\zeta)$.
    \item $\delta_{0} = \exp\left(-\frac{m p \zeta^2}{2}\right)$, and for $\ell \geq 1$, 
    \[\delta_{\ell} = \frac{k^{\ell+1} - 1}{k-1} \exp\left(-\frac{m p \zeta^2}{2}\right)
    + \frac{k^{\ell} - 1}{k-1} \left(k m_1 \cdot \exp\left(-\frac{a m p \zeta^2}{2}\right) + k m_2 \cdot \exp\left(-\frac{a_1 m p \zeta^2}{2}\right) +  m_2 \cdot \exp\left(-\frac{a_2 m p \zeta^2}{2}\right)\right).\]
\end{itemize}

\subsection{The representation of concept hierarchy $\mathcal C$}
\label{sec: representation-lat-low}

As before, we embed the concept hierarchy $\mathcal C$ in the network $\mathcal N$ using a redundant representation.
Thus, for every level $\ell$ concept $c$ of $\mathcal C$, we assume a set $reps(c)$ of $m$ designated neurons in layer $\ell$ of the network; all such sets are disjoint.

We have forward edges from all neurons in each layer to all neurons in the next higher layer, and lateral edges between all neurons in the same layer.
The weights of all edges are in $\{0,1\}$.
The only weight $1$ edges are from $reps$ of child concepts to $reps$ of their parents, and between $reps$ of the same concept (peer $reps$).
However, only a subset of these edges have weight $1$, as specified by the constraints in the \emph{Class Assumption} below.

The Class Assumption aims to achieve incoming potential for all $rep$ neurons of at least $a m$.
We break this down in terms of $reps$ of individual children and peer $reps$, in order to avoid some dependencies between the choice of missing edges and the choice of presented set $B$. 
Specifically, the Class Assumption for concept $c$ says that, for any level $\ell$ concept $c$ in the concept hierarchy, we have a partition of the neurons in $reps(c)$ into two \emph{Classes}, called \emph{Class 1} and \emph{Class 2}, defined as follows:
\begin{itemize}
    \item 
    Neuron $v$ is in Class 1 exactly if, for every child $c'$ of $c$, neuron $v$ has at least $a m$ incoming weight $1$ edges from $reps$ of $c'$. 
    \item 
    Neuron $v$ is in Class 2 exactly if it is not in Class 1, and has the following incoming edges:  (a) for every child $c'$ of $c$, at least $a_1 m$ incoming weight $1$ edges from $reps$ of $c'$ and (b) at least $a_2 m$ incoming weight $1$ edges from $reps$ of $c$ that are in Class $1$. 
\end{itemize}
We assume that this division into Classes is a true partition, that is, it includes all of the neurons in $reps(c)$.

We introduce some notation for the incoming edges for a neuron $v$ in $reps(c)$:
For each child $c'$ of $c$, define $inc(v,c')$ to be the set of all of the neurons that are $reps$ of $c'$ and have weight $1$ edges to $v$.
Also, if $v$ is in Class 2, let $inc(v,c)$ denote the set of all of the neurons that are $reps$ of $c$ in Class 1 and have weight $1$ edges to $v$.  
Then the Class Assumption for $c$ implies that, for every $v$ in Class $1$ and every child $c'$ of $c$, $|inc(v,c')| \geq a m$.
Also, for every $v$ in Class $2$ and every child $c'$ of $c$, $|inc(v,c')| \geq a_1 m$, and for every $v$ in Class $2$, $|inc(v,c)| \geq a_2 m$.
Thus, the Class Assumption for $c$ implies that:
\begin{itemize}
    \item 
For every $v$ in Class 1, $\Sigma_{c'} |inc(v,c')| \geq a k m$, and
\item 
For every $v$ in Class 2, $\Sigma_{c'} |inc(v,c')| + |inc(v,c)| \geq a_1 k m + a_2 m \geq 
a_1 k m + (a - a_1) k m = a k m$.
\end{itemize}

In Section~\ref{sec: guar-lat-low}, we show that each $rep$ in Class 1 has enough firing just from relevant $reps$ at level $\ell-1$ in the network to meet its firing threshold.
And each $rep$ in Class $2$ has enough firing from a combination of relevant $reps$ at level $\ell-1$ and $reps$ in Class $1$ to meet its firing threshold.  

In Section~\ref{sec: learning-lateral}, we try to argue that algorithms similar to those used in the assembly calculus~\cite{DBLP:conf/innovations/Legenstein0PV18,PapadimitriouVempala} ensure the Class Assumption with high probability.  However, the present section is about recognition.  Our job here is simply to show that the Class Assumption suffices to guarantee $(r_1,r_2)$-recognition.

\subsection{The main theorem}

We consider a particular concept hierarchy $\mathcal C$, with concept hierarchy notation as defined in Section~\ref{sec: concepts}.  We define the parameters as in Section~\ref{sec: parameter-values-lat-low} and define the network as in Section~\ref{sec: representation-lat-low}.  Our main result is:

\begin{theorem}
\label{thm: main-lat-low}
Assume that $r_1 \leq a r_2 p(1-\zeta)$ and $r_2 > 0$.
Then $\mathcal N (r_1,r_2)$-recognizes $\mathcal C$.
\end{theorem}

Again, we prove Theorem~\ref{thm: main-lat-low} in two parts, following the definition of $(r_1,r_2)$-recognition.
Theorem~\ref{thm: main-firing-lat-low} expresses the firing guarantee, and Theorem~\ref{thm: main-non-firing-lat-low} expresses the non-firing guarantee.  Theorems~\ref{thm: main-firing-lat-low} and~\ref{thm: main-non-firing-lat-low} together immediately imply Theorem~\ref{thm: main-lat-low}.

\subsection{Survival lemmas}
\label{sec: prob-lemmas-lat}

As before, we have:

\begin{lemma}
\label{lem: reps-for-one-concept-lat-low}
For every concept $c$ in the hierarchy $\mathcal C$, at any level, the probability that the number of surviving neurons in $reps(c)$ is $\leq m p(1-\zeta)$ is at most $\exp\left(-\frac{m p \zeta^2}{2}\right)$.   
\end{lemma}

\begin{lemma}
\label{lem:  reps-for-all-descendants-lat-low}
Consider a particular concept $c$ with $level(c) = \ell$, $0 \leq \ell \leq \ell_{max}$.
Let $A$ be the event that, for some descendant $c'$ of $c$ (possibly $c$ itself), the number of surviving neurons in $reps(c')$ is $\leq m p(1-\zeta)$. 
Then $Pr(A) \leq \frac{k^{\ell+1} - 1}{k-1} \cdot \exp\left(-\frac{m p \zeta^2}{2}\right)$.
\end{lemma}

Now we have some new lemmas, involving the survival of incoming neighbors of a $rep$ neuron $v$. 
Proofs again use Chernoff and union bounds.

\begin{lemma}
\label{lem: one-concept-incoming-class1-lat-low}
Let $c$ be a concept with $level(c) = \ell \geq 1$.
Let $v \in reps(c)$ and suppose that $v$ is in Class 1.
Then for each child $c'$ of $c$, the probability that the number of surviving neurons in $inc(v,c')$ is $\leq amp(1-\zeta)$ is at most $\exp\left(-\frac{a m p \zeta^2}{2}\right)$.
\end{lemma}

\begin{proof}
    By Chernoff, using a mean of $amp$ and concentration parameter of $\zeta$.
    Since the size of $inc(v,c')$ may be larger than $am$, we can consider any subset of $inc(v,c')$ and use a monotonicity argument.
\end{proof}

\begin{lemma}
\label{lem: one-concept-incoming-class2-lat-low}
Let $c$ be a concept with $level(c) = \ell \geq 1$.
Let $v \in reps(c)$ and suppose that $v$ is in Class $2$.
Then:
\begin{enumerate}
    \item 
    For each child $c'$ of $c$, the probability that the number of surviving neurons in $inc(v,c')$ is $\leq a_1 m p(1-\zeta)$ is at most $\exp\left(-\frac{a_1 m p \zeta^2}{2}\right)$.
    \item 
    The probability that the number of surviving neurons in $inc(v,c)$ is $\leq a_2 m p(1-\zeta)$ is at most $\exp\left(-\frac{a_2 m  p \zeta^2}{2}\right)$.
\end{enumerate}
\end{lemma}

\begin{proof}
Both parts use Chernoff with a concentration parameter of $\zeta$.  Part 1 uses a mean of $a_1 m p$ and Part 2 uses a mean of $a_2 m p$.
\end{proof}

\begin{lemma}
\label{lem: all-descendants-incoming-lat-low}
Consider a particular concept $c$ with $level(c) = \ell \geq 1$.
Then:
\begin{enumerate}
    \item 
    Let $A_1$ be the event that, for some descendant $c'$ of $c$ (possibly $c$ itself) with $level(c') \geq 1$, some $v \in reps(c')$ with $v$ in Class 1, and some child $c''$ of $c'$, the number of surviving neurons in $inc(v,c'')$ is $\leq a m p(1-\zeta)$. 
    Then $Pr(A_1) \leq \frac{k^{\ell} - 1}{k-1} \cdot k m_1 \cdot \exp\left(-\frac{a m p \zeta^2}{2}\right)$.
    \item 
    Let $A_2$ be the event that, for some descendant $c'$ of $c$ (possibly $c$ itself) with $level(c') \geq 1$, some $v \in reps(c')$ with $v$ in Class 2, and some child $c''$ of $c'$, the number of surviving neurons in $inc(v,c'')$ is $\leq a_1 m p(1-\zeta)$. 
    Then $Pr(A_2) \leq \frac{k^{\ell} - 1}{k-1} \cdot k m_2 \cdot \exp\left(-\frac{a_1 m p \zeta^2}{2}\right)$.
    \item 
    Let $A_3$ be the event that, for some descendant $c'$ of $c$ (possibly $c$ itself) with $level(c') \geq 1$, and some $v \in reps(c')$ with $v$ in Class 2, the number of surviving neurons in $inc(v,c)$ is $\leq a_2 m p(1-\zeta)$. 
    Then $Pr(A_3) \leq \frac{k^{\ell} - 1}{k-1} \cdot m_2 \cdot \exp\left(-\frac{ a_2 m p \zeta^2}{2}\right)$.
    \item  
    Let $A' = A_1 \cup A_2 \cup A_3$.  Then 
    \[Pr(A') \leq 
    \frac{k^{\ell} - 1}{k-1} \left(k m_1 \cdot \exp\left(-\frac{a m p \zeta^2}{2}\right) + k m_2 \cdot \exp\left(-\frac{a_1 m p \zeta^2}{2}\right) +  m_2 \cdot \exp\left(-\frac{a_2 m p \zeta^2}{2}\right)\right).\]
\end{enumerate}
\end{lemma}

\begin{proof}
The number of descendants of $c$ with level $\geq 1$ is exactly $\frac{k^{\ell}-1}{k-1}$.  
For each descendant $c'$ of $c$ with $level(c') \geq 1$, the number of Class 1 $reps$ is $m_1$, the number of Class 2 $reps$ is $m_2$, and the number of children of $c'$ is $k$.
A union bound, taken over all of these descendants $c'$, $reps$, and children, and Lemma~\ref{lem: one-concept-incoming-class1-lat-low}, yields Part 1.
Two other union bounds and Lemma~\ref{lem: one-concept-incoming-class2-lat-low} yield Parts 2 and 3.
Finally, Part 4 is obtained using yet another union bound, for the events $A_1$, $A_2$, and $A_3$.
\end{proof}

The final lemma combines the results of Lemmas~\ref{lem:  reps-for-all-descendants-lat-low} and~\ref{lem: all-descendants-incoming-lat-low}, for concepts $c$ with $level(c) \geq 1$.

\begin{lemma}
\label{lem: union-of-events-lat-low}
Consider a particular concept $c$ with $level(c) = \ell \geq 1$.
Define $A$ as in the statement of Lemma~\ref{lem:  reps-for-all-descendants-lat-low} and
$A'$ as in the statement of Lemma~\ref{lem: all-descendants-incoming-lat-low}.
Then $Pr(A \cup A') \leq \delta_{\ell} =$
\[\frac{k^{\ell+1} - 1}{k-1} \cdot \exp\left(-\frac{m p \zeta^2}{2}\right)
+ \frac{k^{\ell} - 1}{k-1} \left(k m_1 \cdot \exp\left(-\frac{a m p \zeta^2}{2}\right) + k m_2 \cdot \exp\left(-\frac{a_1 m p \zeta^2}{2}\right) + m_2 \cdot \exp\left(-\frac{a_2 m p \zeta^2}{2}\right)\right).\]
\end{lemma}


To give an idea of the size of $Pr(A \cup A')$ in Lemma~\ref{lem: union-of-events-lat-low}, we assume, as before, that the values of $k$ and $\ell$ are both $4$, the survival probability $p$ for an individual neuron is $\frac{31}{32}$, and the concentration parameter $\zeta$ is $1/4$.
Now we assume that the number $m$ of $reps$ is $640$, and $m_1 = m_2 = 320$.
Also, we assume that $a = 3/4$, $a_1 = 11/16$, and $a_2 = 3/4$.

With these values the probability in this last result is bounded by the sum of four terms (distributing the $\frac{k^{\ell} - 1}{k - 1}$ factor over the three terms in parentheses).  These terms all include the expression $\frac{p \zeta^2}{2}$, which here evaluates to $\frac{(31/32) (1/4)^2}{2})$, which is approximately $.03$.
Then the first term is approximately $4^4 \cdot \exp(-(640)(.03)$, or $256 \cdot \exp(-19.2)$, which is negligible.  The second term is approximately $4^3 \cdot 4 \cdot 320 \cdot \exp(-(480)(.03))$, or $256 \cdot 320 \cdot \exp(-14.4)$, or approximately $.05$.  The third term is approximately $4^3 \cdot 4 \cdot 320 \cdot \exp(-(440) (.03))$, or $256 \cdot 320 \cdot \exp(-13.2)$, or approximately $.15$.  The fourth term is approximately $4^3 \cdot 320 \cdot \exp(-(480)(.03)$, or $64 \cdot 320 \exp(-14.4)$, or approximately $.01$.  The sum of the four terms is approximately $.21$.\footnote{This value seems a bit high.  Of course, we could always increase $m$ to reduce the bound.}

\subsection{Proof of guaranteed firing}
\label{sec: guar-lat-low}

Our main firing theorem, Theorem~\ref{thm: main-firing-lat-low}, still talks about firing of at least $m p (1 - \zeta)$ of the $m$ neurons in the set $reps(c)$.
However, the probability uses the new definition of $\delta_{\ell}$.
Now we assume that a set $B$ of level $0$ concepts is presented at all times $\geq 0$.

\begin{theorem}
\label{thm: main-firing-lat-low}
Assume that a set $B \subseteq C_0$ is presented at all times $\geq 0$.
Let $c$ be a level $\ell$ concept, $0 \leq \ell \leq \ell_{max}$, that is in $supp_{r_2}(B)$.
Then, with probability at least $1 - \delta_{\ell}$, there are at least $m p(1-\zeta)$ neurons in $reps(c)$ that fire at all times starting from some time $\geq 0$.
\end{theorem}

We consider the cases of $\ell = 0$ and $\ell \geq 1$ separately.

\begin{lemma}
\label{lem: 0-case-lat-low}
Assume that a set $B \subseteq C_0$ is presented at all times $\geq 0$.
Let $c$ be a level $0$ concept that is in $supp_{r_2}(B)$.
Then with probability at least $1 - \delta_{0}$, there are at least $m p(1-\zeta)$ neurons in $reps(c)$ that fire at all times $\geq 0$.
\end{lemma}

\begin{proof}
Fix $B \subseteq C_0$, and fix a particular level $0$ concept $c$ that is in $supp_{r_2}(B)$.
Since $c \in supp_{r_2}(B)$ and $level(c) = 0$, we have that $c \in B$.
Since $B$ is presented at all times, all of the surviving neurons in $reps(c)$ fire at all times.
By Lemma~\ref{lem: reps-for-one-concept-lat-low}, with probability at least $\delta_0$, the number of surviving neurons in $reps(c)$ is at least $mp(1-\delta)$.  
Therefore, with probability at least $\delta_0$, there are at least $mp(1-\delta)$ neurons in $reps(c)$ that fire at all times $\geq 0$.
\end{proof}

\begin{lemma}
\label{lem: nonzero-case-lat-low}
Assume that a set $B \subseteq C_0$ is presented at all times $\geq 0$.
Let $c$ be a level $\ell$ concept, $1 \leq \ell \leq \ell_{max}$, that is in $supp_{r_2}(B)$.
Then with probability at least $1 - \delta_{\ell}$, there are
at least $m p(1-\zeta)$ neurons in $reps(c)$ that fire at all times $\geq 2 \ell$.
\end{lemma}

\begin{proof}
Fix $B$, and fix a particular level $\ell$ concept $c$, $1 \leq \ell \leq \ell_{max}$, that is in $supp_{r_2}(B)$.
Define event $A$ as in Lemma~\ref{lem:  reps-for-all-descendants-lat-low} and event $A'$ as in 
Lemma~\ref{lem: all-descendants-incoming-lat-low}. 
Then by Lemma~\ref{lem: union-of-events-lat-low}, 
$Pr(A \cup A') \leq \delta_{\ell}$.

For the rest of the proof, we condition on the event $\overline{A \cup A'}$.  
The proof uses the following Claim:

\begin{claim}
\label{claim: firing-lat-low}
If $c'$ is a descendant of $c$ with $level(c') = \ell'$, $0 \leq \ell' \leq \ell$, with $c' \in supp_{r_2}(B)$, then:
\begin{enumerate}
\item  If $\ell' \geq 1$ then every Class 1 neuron $v$ in $reps(c')$ has its threshold met for all times $\geq 2 \ell' - 1$.
\item  If $\ell' \geq 1$ then every Class 2 neuron $v$ in $reps(c')$ has its threshold met for all times $\geq 2 \ell'$.
\item  At least $m p(1-\zeta)$ of the neurons in $reps(c')$ fire at all times $\geq 2 \ell'$.   
\end{enumerate}
\end{claim}

\noindent
\emph{Proof of Claim~\ref{claim: firing-lat-low}:} \\
By induction on $\ell'$, using two base cases. 

\emph{Base}, $\ell' = 0$:
Parts 1 and 2 are vacuous.
For Part 3, since $c' \in supp_{r_2}(B)$ and $level(c') = 0$, we have that $c' \in B$.
Since $B$ is presented at all times, all of the surviving neurons in $reps(c')$ fire at all times.  Since we are conditioning on $\overline{A \cup A'} \subseteq \overline{A}$, the number of surviving neurons in $reps(c')$ is at least $mp(1-\zeta)$.
Therefore at least $mp(1-\zeta)$ of the neurons in $reps(c')$ fire at all times $\geq 0 = 2 \cdot 0$, as needed. 

\emph{Base}, $\ell' = 1$:
We show Parts 1 and 2 individually.  Together they imply that every neuron $v$ in $reps(c')$ has its threshold met at all times $\geq 2 = 2 \ell'$.  Then Part 3 will follow since we are conditioning on $\overline{A \cup A'} \subseteq \overline{A}$.

To show Part 1, we fix any Class 1 neuron $v$ and show that its threshold is met for all times $\geq 1 = 2 \ell' - 1$.
Because $c'$ is supported by $B$, $c'$ has at least $r_2k$ children that are supported by $B$, and therefore are in $B$.  Consider each supported child $c''$ of $c'$ individually.
Since we are conditioning on $\overline{A \cup A'} \subseteq \overline{A'}$, we know that at least $a m p(1-\zeta)$ neurons in $inc(v,c'')$ survive.
Since $level(c'') = 0$ and $B$ is presented at all times, all of these neurons fire at all times.
Taking into account all of the supported children $c''$ of $c'$, we have 
at least $a r_2 k m p(1 - \zeta)$ weight $1$ connections to $v$ from $reps$ of children of $c'$ that fire at all times, yielding an incoming potential to $v$ of at least $a r_2 k m p(1 - \zeta)$.
This meets $v$'s firing threshold for all times $\geq 1$.

To show Part 2, we fix any Class 2 neuron $v$ and show that its threshold is met for all times $\geq 2 = 2 \ell'$.
Because $c'$ is supported by $B$, $c'$ has at least $r_2 k$ children that are supported by $B$, and therefore are in $B$.  Consider each supported child $c''$ of $c'$ individually.
Since we are conditioning on $\overline{A \cup A'} \subseteq \overline{A'}$, we know that at least $a_1 m p(1-\zeta)$ neurons in $inc(v,c'')$ survive.
Since $level(c'') = 0$ and $B$ is presented at all times,  all of these neurons fire at all times.
Taking into account all of the supported children $c''$ of $c'$, we have 
at least $a_1 r_2 k m p(1 - \zeta)$ weight $1$ connections to $v$ from firing $reps$ of children of $c'$, yielding an incoming potential to $v$ of at least $a_1 r_2 k m p(1 - \zeta)$.

But this may not be enough to meet the meet the firing threshold for $v$; for this, we must also consider connections from the $reps$ of $c'$ that are in Class 1.
Since we are conditioning on $\overline{A \cup A'} \subseteq \overline{A'}$, we have that at least $a_2 m p(1-\zeta)$ neurons in $inc(v,c')$ survive.
Since these are all in Class 1, Part 1 implies that all of their firing thresholds are met, and so they fire at all times $\geq 1$.
Since the weights for the edges from these neurons to $v$ are all $1$, this gives a total of $a_2 m p(1-\zeta)$ potential contributed to $v$ from these peers of $v$ in Class 1, for all times $\geq 2$.
Then the total potential incoming to $v$ for all times $\geq 2$ is at least $a_1 r_2 k m p(1-\zeta) + a_2 m p(1-\zeta) \geq  a_1 r_2 k m p(1-\zeta) + (a - a_1) k m p(1-\zeta) \geq a r_2 k m p(1-\zeta)$.
This meets $v$'s firing threshold for all times $\geq 2$. 

\emph{Inductive step}, $2 \leq \ell' \leq \ell$:  
This is similar to the second base step.
We show Parts 1 and 2 individually.  Together they imply that every neuron $v$ in $reps(c')$ has its threshold met at all times $\geq 2 \ell'$.
Then Part 3 follows since we are conditioning on $\overline{A \cup A'} \subseteq \overline{A}$.

To show Part 1, we fix any Class $1$ neuron $v$ and show that its threshold is met for all times $\geq 2 \ell'-1$.
Because $c'$ is supported by $B$, it has at least $r_2 k$ children that are supported by $B$.
Consider each supported child $c''$ of $c'$ individually.
Since we are conditioning on $\overline{A \cup A'} \subseteq \overline{A'}$, we have that at least $amp(1-\zeta)$ neurons in $inc(v,c'')$ survive.
Since $\ell' -1 \geq 1$, by the inductive hypothesis, Parts 1 and 2, these all have their thresholds met for all times $\geq 2 (\ell'-1)$, and therefore they fire at all times $\geq 2 (\ell'-1)$
Taking into account all of the supported children $c''$ of $c'$, we have 
at least $a r_2 k m p(1-\zeta)$ weight $1$ connections to $v$ from $reps$ of children of $c'$ that fire at all times $\geq 2 (\ell'-1)$.
This meets $v$'s firing threshold for all times $\geq 2 \ell' - 1$. 

To show Part 2, we fix any Class 2 neuron $v$ and show that its threshold is met for all times $\geq 2 \ell'$.
Because $c'$ is supported by $B$, it has at least $r_2 k$ children that are supported by $B$.
Consider each supported child $c''$ individually.
Since we are conditioning on $\overline{A \cup A'} \subseteq \overline{A'}$, we have that at least $a_1 m p(1-\zeta)$ neurons in $inv(v,c'')$ survive.
Since $\ell' \geq 1$, by the inductive hypothesis, Parts 1 and 2, these all have their thresholds met for all times $\geq 2(\ell'-1)$, and therefore they fire at all times $\geq 2(\ell'-1)$.  
Taking into account all of the supported children $c''$ of $c'$, we have at least $a_1 r_2 k m p(1-\zeta)$ weight $1$ connections to $v$ from $reps$ of children of $c'$ that fire at all times $\geq 2 (\ell'-1)$, yielding an incoming potential to $v$ for all times $\geq 2 \ell' - 1$ of at least $a_1 r_2 k m p(1-\zeta)$.

To meet the firing threshold for $v$, we also consider the $reps$ of $c'$ that are in Class 1.
Arguing as in the base case for $\ell' = 1$, we get a total of $a_2 m p(1-\zeta)$ potential contributed to $v$ from these peers of $v$ in Class 1, for all times $\geq 2 \ell'$.
Then the total potential incoming to $v$ for all times $\geq 2\ell'$ is at least 
$a_1 r_2 k m p(1-\zeta) + a_2 m p(1-\zeta) \geq a_1 r_2 k m p(1-\zeta) + (a - a_1) k m p(1-\zeta)\geq  a r_2 k m p(1-\zeta)$.
This meets $v$'s firing threshold for all times $\geq 2 \ell'$. \\
\emph{End of proof of Claim~\ref{claim: firing-lat-low}.}

Instantiating Part 3 of Claim~\ref{claim: firing-lat-low} with $c' = c$ yields the lemma, taking into account that $Pr(A \cup A') \leq \delta_{\ell}$. 
\end{proof} %

\begin{proof} (Of Theorem~\ref{thm: main-firing-lat-low}):  Follows from Lemmas~\ref{lem: 0-case-lat-low} and~\ref{lem: nonzero-case-lat-low}. 
\end{proof}

\subsection{Proof of guaranteed non-firing}

And now we prove our main non-firing theorem.

\begin{theorem}
\label{thm: main-non-firing-lat-low}
Assume that $r_1 \leq a r_2 p(1-\zeta)$ and $r_2 > 0$.
Assume that a set $B \subseteq C_0$ is presented at all times.  
Let $c$ be a level $\ell$ concept, $0 \leq \ell \leq \ell_{max}$, that is not in $supp_{r_1}(B)$.
Then none of the neurons in $reps(c)$ fire at any time.
\end{theorem}

\begin{proof} (Sketch:)
As before, we ignore the failures for the non-firing part.  Now, because we want to consider the most favorable situation for the recognition algorithm, we assume all-to-all connectivity from $reps$ of children to $reps$ of parents, and also among $reps$ of the same concept.
Under these conditions, all $reps$ of all concepts at levels $\geq 1$ are in Class 1.
Then the proof again follows arguments like those in~\cite{DBLP:conf/sirocco/LynchM23}.
\end{proof}

\section{Learning in Feed-Forward Networks}
\label{sec: learning}

In this section and the next, we describe how the representations discussed in the previous three sections might be learned.  We do not present our learning algorithms in detail, but just discuss them at a high level, relying freely on results from~\cite{LM21} and~\cite{DBLP:conf/innovations/Legenstein0PV18} for some of the descriptions and arguments.

We focus here on the problem of learning a particular concept $c$ and its descendants, rather than learning the entire concept hierarchy $\mathcal C$.
The learning strategies we consider all proceed bottom-up, identifying the set $reps(c')$ for a descendant $c'$ of $c$, and fixing their incoming edge weights, only after the $reps$ and weights for the children of $c'$ has already been learned.

We assume that the network contains forward edges from all layer $\ell - 1$ neurons to all layer $\ell$ neurons.  In addition, for networks with lateral edges, we have edges between all pairs of neurons in the same layer.
The initial weights of the edges will be specified as part of each algorithm description.  After a learning algorithm completes, the weights of all edges will be either $0$ or $1$.
The only edges that may have weight $1$ after learning is completed are forward edges from $reps$ of children to $reps$ of their parents, and lateral edges between $reps$ of the same concept.
We assume here that the chosen $reps(c)$ neurons record this status in their state, by setting a special $rep$ variable, whose value starts as $\bot$, to $c$.

To keep things simple, we avoid considering failures during learning.\footnote{It is not entirely clear what types of failures should be considered.  Initial failures of layer $0$ $rep$ neurons do not seem reasonable, because then we would entirely miss learning some parts of the concept hierarchy.  We might consider initial failures of higher-layer neurons, or failures of different layer $0$ neurons at the start of each new training instance; this latter type of failure is similar to what we assumed in the "noisy learning" algorithm in~\cite{LM21}.  
But at any rate, we avoid this for now and leave it for future work.}

In the remainder of this section, we consider feed-forward networks, in both the high-connectivity and low-connectivity cases.
In Section~\ref{sec: learning-lateral} we consider networks that also contain lateral edges.

In feed-forward networks, in both the high-connectivity and low-connectivity cases, the $reps$ for a concept $c$ can be selected and their incoming weights can be learned using  variants of the noise-free learning algorithm in~\cite{LM21}.

\subsection{Feed-forward networks with high connectivity}
\label{sec: learning-ff}

In this case, as we described earlier, we assume that the network $\mathcal N$ has total connectivity from each layer to the next-higher layer, and no lateral edges.
Prior to learning, we assume that all the edge weights of forward edges have the same small value, for example, $\frac{1}{k^{\ell_{max}} + 1}$ as in~\cite{LM21}.
As in~\cite{LM21}, we assume that the thresholds are known ahead of time and are fixed at the same value for all neurons in layers $\geq 1$. 
Since we are not now considering partial information or failures, we assume that the threshold is simply $k m$.\footnote{This threshold value is higher than the threshold of $r_2 k m p(1-\zeta)$ that is used in the main recognition result, Theorem~\ref{thm: main}.
For now, we will just assume that the firing threshold can be different for learning and recognition, which is not implausible.
If we modify the learning algorithm so that it allows some partial information or failures, then its threshold would have to be reduced.  We leave this for future work.}

For learning a concept $c$, we generally follow the bottom-up noise-free learning algorithm of~\cite{LM21}, learning the descendants of $c$ level-by-level starting with level $1$.
Consider a particular descendant $c'$ of $c$ with $level(c') = \ell'$, and suppose that the children of $c'$ have already been learned. Then $leaves(c')$ get presented, and cause all the $reps$ of children of $c'$ to fire.
These provide incoming potential to the neurons at layer $\ell'$.
In~\cite{LM21}, at this point in the algorithm, a Winner-Take-All mechanism is used to select a single $rep$ neuron in layer $\ell'$, namely, one with the highest incoming potential (breaking ties according to some default).
Here, instead of using a simple Winner-Take-All, the learning algorithm uses an $m$-Winner-Take-All, which selects the $m$ available neurons in layer $\ell'$ that have the highest incoming potential.  Note that these neurons need not have fired---they just need to have high incoming potential.

The algorithm then engages all of the selected neurons in learning, and adjusts their incoming weights in several steps, strengthening the weights of edges from $reps$ of children of $c'$ to $reps$ of $c'$, and weakening other incoming edges of $reps$ of $c'$.
Note that the final weights in~\cite{LM21} are scaled and approximate, approaching values $0$ or $\frac{1}{\sqrt{k}}$ in the limit.\footnote{To approach $1$ instead of $\frac{1}{\sqrt{k}}$, we can use simple scaling ideas as in~\cite{DBLP:conf/sirocco/LynchM23}.}
Here we assume a modified Hebbian-style learning rule that actually reaches final values of $0$ or $1$.\footnote{Such a modified rule might involve slight adjustment to weights that are "very close" to their targets so that they actually reach these values.}
We do this because this matches up better with our recognition results in Section~\ref{sec: recog-ff}. 

We state our correctness theorem somewhat informally.  
We say that a network (one that results from learning) \emph{correctly represents} a concept $c$ with $level(c) \geq 1$ if it contains 
$reps$ and weights that satisfy the requirements in Section~\ref{sec: representation} for $c$ and its descendants.\footnote{This is not only informal, but also not quite right because the definition in Section~\ref{sec: representation} is for the network representing the entire concept hierarchy $\mathcal C$, whereas we are concerned here only for $c$ and its descendants.  But it is straightforward to modify the earlier definitions so that they apply to this smaller hierarchical structure.}


\begin{theorem}
\label{thm: learning-2}
Let $c$ be any particular concept in $\mathcal C$, with $level(c) = \ell, 1 \leq \ell \leq \ell_{max}$.
Then our learning algorithm results in a network that correctly represents $c$. 
\end{theorem}

\begin{proof}
Analogous to the proof for the noise-free learning algorithm in~\cite{LM21}.
\end{proof}

Theorem~\ref{thm: learning-2} says that the learning algorithm yields a correct representation for $c$ and its descendants.  We can also consider correctness of a combination of successful learning of $c$ and successful recognition of $c$ after learning:

\begin{corollary}
\label{cor: combined-learning-and-recognition-1}
Let $c$ be any particular concept in $\mathcal C$, with $level(c) = \ell, 1 \leq \ell \leq \ell_{max}$.
The learning algorithm for $c$ yields a network that satisfies the following property. 
Let $B \subseteq C_0$ such that $c$ is in $supp_{r_2}(B)$.
Then with probability at least $1 - \delta_{\ell}$, at least $m p(1-\zeta)$ of the neurons in $reps(c)$ fire at time $\ell$.
\end{corollary}

\begin{proof}
By Theorem~\ref{thm: learning-2} and Theorem~\ref{thm: main-firing}.
The first of these is guaranteed to yield a correct representation, and the second of these relies on the representation to get recognition with probability at least $1 - \delta_{\ell}$.
\end{proof}

\subsection{Feed-forward networks with low connectivity}
\label{sec: learning-ff-low}

For the case of feed-forward networks with low connectivity, we use a different assumption about the initial network from what we used in Section~\ref{sec: learning-ff}.
Instead of assuming total connectivity between successive layers with nonzero weight edges, we assume that the connections are chosen randomly.
Thus, we assume the following new parameter:
\begin{itemize}
    \item $p' \in [0,1]$, representing the probability of including each edge.  Specifically, for each edge from a layer $\ell$ to the next-higher layer $\ell + 1$, with probability $p'$, we assign a small weight to the edge, and otherwise we assign weight $0$.
    \end{itemize}

For learning a concept $c$, the learning algorithm proceeds by learning the descendants of $c$ bottom-up, as in Section~\ref{sec: learning-ff}.
We modify the thresholds to accommodate the missing edges. Thus, instead of using a threshold of $k m$, we use $a k m$, where $a$ is the coefficient that is assumed in Section~\ref{sec: recog-ff-low}.
Again, for each descendant $c'$ of $c$ with $level(c') = \ell'$, we use an $m$-Winner-Take-All mechanism to select the $m$ available neurons in layer $\ell'$ with the highest incoming potential from all $km$ $reps$ of all children of $c'$.  We increase the nonzero weights of incoming edges from $reps$ of children of $c'$ to $reps$ of $c'$, to $1$.\footnote{
When we increase the weights of edges from $reps$ of children of $c'$ to $reps$ of $c'$, we avoid increasing the weights of edges whose weights are $0$.   Thus, we are treating those edges whose weights were initially set to $0$ as if they did not exist.}

For this to work correctly, we must reconcile a fundamental difference between the operation of the learning algorithm and the conditions required for the representation, in Section~\ref{sec: representation-low}.
For any particular descendant $c'$ of $c$, our goal is that, for every chosen $v \in reps(c')$, with high probability, each of the $k$ children of $c'$ individually should provide at least $am$ weight $1$ edges from $reps(c'')$ to $v$.
However, the algorithm chooses the $m$ $reps$ of $c'$ based on having the highest total incoming potential from all $k m$ $reps$ of all children of $c'$.  
That means that they have the highest total number of nonzero weight incoming edges, from all of these $km$ $reps$ combined.

To reconcile this difference, we assume the following new parameters:
\begin{itemize}
    \item $b \in (a,1]$; for each $v$, the learning algorithm should yield at least $b k m$ total weight $1$ incoming edges to $v$ from $reps$ of all children of $c'$.
    \item $\theta^{learn} \in [0,1]$, a small bound on the probability that the learning algorithm does not yield this total.
    \item $\theta^{bridge} \in [0,1]$, a small probability to bridge the gap between the bound of $b k m$ total weight $1$ incoming edges achieved by the learning algorithm and the needed individual bounds of $a m$ weight $1$ incoming edges for the separate children of $c'$.
\end{itemize}
Let $\theta = \theta^{learn} + \theta^{bridge}$, and $\theta_{\ell} = \frac{k^{\ell}-1}{k-1} \theta$,
 
These parameters are required to satisfy the following two constraints. 
We state them in terms of simplified experiments.  \\

\noindent
\textbf{Constraint 1:}
Consider two layers, $\ell-1$ and $\ell$.
Consider any particular set $S$ of exactly $k m$ neurons in layer $\ell-1$.
Suppose that each edge from layer $\ell-1$ to layer $\ell$ is selected independently, with probability $p'$.  
Then with probability at least $1 - \theta^{learn}$, there are at least $m$ neurons in layer $\ell$, each of which has at least $b k m$ incoming edges from neurons in $S$. \\

\noindent
\textbf{Constraint 2:}
Consider the experiment of choosing edges independently with probability $p'$, from a pool of $k$ groups of $m$ edges apiece.
Let $B$ be the event that at least $b k m$ edges are chosen, in total.
Let $A$ be the event that, in each of the $k$ separate groups of $m$ edges, at least $a m$ are chosen.
Then $Pr(A|B) \geq 1 - \theta^{bridge}$. \\

To match these experiments up with the actual learning algorithm, we simplify by assuming that the concept $c$ (and all of its descendants) are chosen first, and then the edges are chosen randomly.  This is because we want to avoid having the choice of concept $c$ depend on the choice of missing edges, which would complicate analysis. 

We consider learning the $reps$ of one descendant $c'$ of $c$, after learning the $reps$ of all the children of $c'$.

For Constraint 1, we let $S$ be the set of all $k m$ $reps$ of children of $c'$.  Then Constraint 1 says that, with high probability, each of the $m$ layer $\ell$ neurons with the highest potential has at least $b k m$ nonzero-weight incoming edges from $reps$ of children of $c'$.  
This implies that, after learning, each of the $m$ chosen $reps$ of $c'$ has at least $b k m$ weight $1$ incoming edges.\footnote{Constraint 1 certainly holds for sufficiently large values of $p'$, $b$, and the number $n$ of potential level $0$ concepts.
We leave it for future work to characterize the satisfying values precisely.}

For Constraint 2, we fix any particular $v \in reps(c')$.  We consider the $k$ groups of $m$ edges each, from the $reps$ of the $k$ children of $c'$ to $v$.
Then $B$ is the event that $v$ has at least $b k m$ incoming nonzero weight edges from all the $reps$ of children of $c'$, in total.
And $A$ is the event that $v$ has at least $a m$ incoming nonzero weight edges from each group of $m$ edges.
Constraint 2 says that, with high probability, given that $v$ has at least $b k m$ incoming nonzero weight edges from all the $reps$ of children of $c$, in total, then $v$ has at least $a m$ incoming nonzero weight edges from the $reps$ of each child.
Using Constraint 1, this is enough to show that, after learning, $v$ has at least $a m$ weight $1$ incoming edges from the $reps$ of each child.  This is what is needed for the representation in Section~\ref{sec: representation-low}.\footnote{
 We again leave it for future work to characterize the parameter values for which Constraint 2 holds.
 Here we simply note some relevant facts.
 In Constraint 2, $Pr(B)$ is given by by the upper tail of the binomial distribution $b(k m, p')$, based on the value being $\geq b k m$.
 And $Pr(A) = Pr(A_1)^k$, where $A_1$ is the event that, out of one group, say the first group, at least $a m$ are chosen.
 Here $Pr(A_1)$ is given by the upper tail of the binomial distribution $b(m,p')$, based on the value being $\geq a m$.
 Note that $A \subseteq B$, since $a \leq b$.
 So we have $Pr(A | B) = \frac{Pr(A \cap B)}{Pr(B)} = \frac{Pr(A)}{Pr(B)}$.}

Once again, we state our correctness theorem somewhat informally.
We say that a network (one that results from learning) \emph{correctly represents} a concept $c$ with $level(c) \geq 1$ if it contains $reps$ and weights that satisfy the requirements in Section~\ref{sec: representation-low} for $c$ and its descendants.
Now we obtain:

\begin{theorem}
\label{thm: learning-low-2}
Assume that the parameters ($n$, $k$, $m$, $p'$, $a$, $b$, $\theta^{learn}$, $\theta^{bridge}$) satisfy Constraints 1 and 2.
Let $c$ be any particular concept in $\mathcal C$, with $level(c) = \ell$, $1 \leq \ell \leq \ell_{max}$.
Then with probability at least $1 - \theta_{\ell}$, our learning algorithm results in a network that correctly represents $c$. 
\end{theorem}

Theorem~\ref{thm: learning-low-2} says that with high probability, the learning algorithm yields a correct representation for $c$ and its descendants.
We can also consider correctness of a combination of successful learning of $c$ and successive recognition of $c$ after learning:

\begin{corollary}
\label{cor: combined-learning-and-recognition}
Let $c$ be any particular concept in $\mathcal C$, with $level(c) = \ell, 1 \leq \ell \leq \ell_{max}$.
The learning algorithm for $c$, with probability at least $1 - \theta_{\ell}$, yields a network that satisfies the following property. 
Let $B$ be a set of level $0$ concepts such that $c$ is in $supp_{r_2}(B)$.
Then with probability at least $1 - \delta_{\ell}$, at least $m p(1-\zeta)$ of the neurons in $reps(c)$ fire at time $\ell$.
\end{corollary}

\begin{proof}
By Theorem~\ref{thm: learning-low-2} and Theorem~\ref{thm: main-firing-low}.
The first of these yields a correct representation with probability at least $1 - \theta_{\ell}$ and the second of these relies on the representation to achieve recognition with probability at least $1 - \delta_{\ell}$.
\end{proof}

We can also consider a combined algorithm, one that first learns concept $c$, then attempts to recognize $c$ based on a presented set $B$ such that $c$ is in $supp_{r_2}(B)$.
The probability that recognition succeeds in this combined algorithm is at least 
$(1 - \theta_{\ell})(1 - \delta_{\ell})$.

\section{Learning in networks with lateral edges}
\label{sec: learning-lateral}

Now we consider learning for networks that include lateral edges.
This work is generally inspired by the earlier work on the assembly calculus in~\cite{DBLP:conf/innovations/Legenstein0PV18,PapadimitriouVempala}.
Specifically, our learning problem is similar to learning a new concept formed from previously-learned concepts using the Merge/Join operation. 
Here we are equating our layers with the "areas" used in~\cite{DBLP:conf/innovations/Legenstein0PV18,PapadimitriouVempala}.

We start in Section~\ref{sec: our-version-assembly-calculus} with an algorithm that is fairly directly based on the earlier work.
Since that is somewhat complicated, we give a simplified version in Section~\ref{sec:  new-learning-algorithm}.
We formulate the main correctness result for this version as a conjecture and give preliminary ideas for a proof.

The goal of these algorithms is to identify the $reps$ and fix their incoming weights, in such a way as to satisfy the requirements for a representation in Section~\ref{sec: representation-lat-low}.  This includes our Class Assumption.

\subsection{Algorithm based on the assembly calculus}
\label{sec: our-version-assembly-calculus}

Our algorithm is based on the $Project$ algorithm  in~\cite{DBLP:conf/innovations/Legenstein0PV18}.
This algorithm learns representations of concepts from sensory input.
Here we are using a $k$-way $Merge$, not a simple $Project$, but the ideas are similar.

Recall from Section~\ref{sec: learning} that our method of selecting new $reps$ for a concept $c$ in a feed-forward network involves a single-step procedure, in which an $m$-WTA chooses the $m$ available neurons in the layer corresponding to $level(c)$ having the highest incoming potential from the $reps$ of children of $c$.
In contrast, for networks with lateral edges, the algorithm in~\cite{DBLP:conf/innovations/Legenstein0PV18} uses a more elaborate multi-step procedure involving repeated uses of an $m$-WTA.

In the algorithm in~\cite{DBLP:conf/innovations/Legenstein0PV18}, the "sensory neurons" fire at every step, contributing potential to neurons in the appropriate target area. 
In step $1$, the algorithm uses an $m$-WTA to choose the $m$ available neurons in the target area with the highest incoming potential from sensory neurons.  
We can regard these as initial candidates to become $reps$ of $c$.
Then the algorithm executes additional steps $2,\ldots$, in each of which:

\begin{enumerate}[(a)]
    \item The $m$ current candidates fire and contribute potential to other neurons in the area.  The neurons in the area also continue to receive potential from the always-firing sensory neurons.
    \item A new set of $m$ available neurons is chosen.  These are the neurons with the highest incoming potential based on a combination of potential from sensory neurons and from other $reps$ of $c$ that are firing.
    \item The edges that contributed to the selection of the new set of candidate neurons have their weights increased by a simple Hebbian-style rule. 
    Other edges incoming to the new candidate neurons have their weights decreased, in order to normalize the total incoming weights to each of these neurons.
\end{enumerate}
The presentation in~\cite{DBLP:conf/innovations/Legenstein0PV18} analyzes the convergence behavior of this algorithm.

To express this in terms of our model, we assume initial random connections.
Formally, for each edge from a layer to the next-higher layer, and each edge within the same layer, with a high probability $p'$, we assign a small weight to the edge, and otherwise we assign weight $0$.
For our threshold, we use $a k m$.
Our algorithm essentially follows the procedure described just above, but stopping after some fixed finite number of steps, say $t$ steps.
Here the sensory neurons of~\cite{DBLP:conf/innovations/Legenstein0PV18} correspond to the $reps$ of the children of $c$.
The final result is a set of $m$ neurons, which have been put into the set at various steps, some based just on firing of $reps$ of children of $c$, and some based also on firing of peer $reps$ that were previously put into the set.

We repeat the algorithm description using our terminology.
In one stage of our algorithm, for learning a descendant $c'$ of $c$, the $reps$ of children of $c'$ fire at every step, contributing potential to neurons in layer $\ell' = level(c')$.
In step $1$, the algorithm uses an $m$-WTA to choose the $m$ available neurons in layer $\ell'$ with the highest incoming potential from $reps$ of children of $c'$.  
Then the algorithm executes additional steps $2,\ldots$, in each of which:

\begin{enumerate}[(a)]
    \item The $m$ current candidates fire and contribute potential to other neurons in layer $\ell'$.  The neurons in layer $\ell'$ also continue to receive potential from the always-firing $reps$ of children of $c'$.
    \item A new set of $m$ available neurons is chosen.  These are the neurons with the highest incoming potential based on a combination of potential from $reps$ of children of $c'$ and from other $reps$ of $c'$ that are firing.
    \item The edges that contributed to the selection of the new set of candidate neurons have their weights increased by a simple Hebbian-style rule. 
    Other edges incoming to the new candidate neurons have their weights decreased, in order to normalize the total incoming weights to each of these neurons.
\end{enumerate}

We state our correctness claim somewhat informally.
At this point, it is only a conjecture, since we have not produced a proof.

We say that a network (one that results from learning) \emph{correctly represents} a concept $c$ with $level(c) = \ell$, $1 \leq \ell \leq \ell_{max}$ if it contains $reps$ and weights that satisfy the requirements in Section~\ref{sec: representation-lat-low} for $c$ and its descendants.  This includes the Class Assumption.

Define the failure probability $\theta_{\ell}$ to be $\frac{k^{\ell}-1}{k-1} \theta$, where $\theta$ is a small value to be determined.

\begin{conjecture}
\label{thm: learning-lateral-2}
Let $c$ be any particular concept in $\mathcal C$, with $level(c) = \ell, 1 \leq \ell \leq \ell_{max}$.
Then with probability at least $1 - \theta_{\ell}$, our rewrite of the assembly calculus learning algorithm results in a network that correctly represents $c$. 
\end{conjecture}

\subsection{A simplified algorithm}
\label{sec:  new-learning-algorithm}

The algorithm of the previous section seems to satisfy Conjecture~\ref{thm: learning-lateral-2}, with appropriate values of the parameters,
but seems difficult to analyze. The main complication for analysis is that the candidate set may change at each step.  So in this section, we suggest a simplified version of that algorithm, in which the candidate set does not change.  We have not yet compared the behavior of this algorithm to that of the algorithm of Section~\ref{sec: our-version-assembly-calculus}.

Here we describe learning of a single concept $c$ with $level(c) = \ell$, $\ell \geq 1$, assuming that its children have already been learned.  The full algorithm operates bottom-up, as usual, learning the concepts at each level only after the concepts at the previous levels have already been learned.

We use the following parameters:
\begin{itemize}
    \item $p' \in [0,1]$, the probability of including each edge.  Specifically, for each forward edge or lateral edge, with probability $p'$, we assign a small weight to the edge, and otherwise we assign weight $0$.
    \item $b \in [a,1]$, where $a$ is the coefficient that is assumed in Section~\ref{sec: representation-lat-low}; $b$ is a coefficient to be achieved in the learning algorithm.  As before, the purpose of $b$ is to bridge the gap between a threshold for overall potential and a threshold for potential from $reps$ of individual children. 
    \item $b_1 \in [a_1,1]$, a new coefficient corresponding to $a_1$ in Section~\ref{sec: parameter-values-lat-low}.
    \item $m_1$, $m_2$, where $m_1 + m_2 = m$; these are the sizes of the two Classes.
    \item $\theta$, a small probability for failing to learn; define $\theta_\ell = \frac{k^{\ell} - 1}{k-1} \theta$.
\end{itemize}

The new algorithm operates in two phases, identifying the Class 1 neurons in Phase 1 and the Class 2 neurons in Phase 2.
We use a threshold of $b k m$.

\begin{itemize}
    \item \emph{Phase 1:}  
    Trigger all of the $reps$ of children of $c$ to fire (this can be done by triggering all the $reps$ of leaves of $c$ to fire).
    The algorithm identifies the neurons in Class 1, based on the total incoming potential from the $k m$ $reps$ of children of $c$.  Specifically, an $m_1$-WTA  chooses the $m_1$ available neurons with the highest incoming potential as the Class 1 $reps$.  
    Then the algorithm engages all of the selected $reps$ in learning, and adjusts their incoming weights, either in several steps or all at once.\footnote{In the case without failure/noise during learning, it should be possible to increase the weights in one, or a few steps instead of working in small increments.  The same holds for our algorithms in Section~\ref{sec: learning}.} The algorithm increases the weights to $1$ for all the edges incoming to the chosen $reps$ that contributed to the incoming potential, and decreases the weights to $0$ for all other edges incoming to these $reps$.
    \item \emph{Phase 2:}   
    Trigger all of the $reps$ of children of $c$ to fire, as in Phase 1. 
    Now also trigger all of the $reps$ of neurons in Class 1 to fire (again, this can be done by triggering all the $reps$ of leaves of $c$ to fire).
    The algorithm identifies the neurons in Class 2, based on the total incoming potential from the $k m$ $reps$ of children of $c$ plus the $m-1$ other $reps$ of $c$.
    Specifically, an $m_2$-WTA chooses the $m_2$ available neurons with the highest incoming potential as the Class 2 $reps$.
    Then the algorithm engages all of the selected $reps$ in learning, and adjusts their incoming weights, either in several steps or all at once. The algorithm increases the weights to $1$ for all the edges incoming to the chosen $reps$ that contributed to the incoming potential, and decreases the weights to $0$ for all other edges incoming to these $reps$.
\end{itemize}

Thus, in this section, we are basing our choices on total incoming potential, whereas our requirements in Section~\ref{sec: representation-lat-low} involve potential from each individual child.  As before, we need to introduce new, higher coefficients, $b$ and $b_1$, bridge the gaps.

\begin{conjecture}
\label{thm: learning-lateral-3}
Assume that the parameters satisfy appropriate constraints.
Let $c$ be any particular concept in $\mathcal C$, with $level(c) = \ell, 1 \leq \ell \leq \ell_{max}$.
Then with probability at least $1 - \theta_{\ell}$, our new learning algorithm results in a network that correctly represents $c$. 
\end{conjecture}

We suspect that this analysis should be tractable, using ideas like those in Section~\ref{sec: learning-ff-low}.  But this remains to be done.

\begin{proof}
(\emph{Preliminary ideas:})
We consider learning one concept $c$, assuming that its children have already been learned.  The full result follows by applying this result to every descendant of $c$.

Proceeding as in Section~\ref{sec: learning-ff-low}, we assume that values of the parameters satisfy constraints sufficient to achieve high numbers of incoming edges:  
$b k m$ edges from $reps$ of children of $c$ for each of the $m_1$ Class 1 $reps(c)$ neurons, 
$b_1 k m$ edges from $reps$ of children of $c$ for each of the $m_2$ Class 2 $reps(c)$ neurons, and 
$a_2 m$ edges from $reps$ of $c$ for each of the $m_2$ Class 2 $reps(c)$ neurons.
The discrepancies between the coefficients $a$ and $b$, and between $a_1$ and $b_1$, serve to bridge the gaps between the requirements on overall number of incoming edges and the requirements on number of incoming edges from $reps$ of individual children.

It remains to define the needed constraints precisely, and characterize the values of the many parameters that satisfy the constraints.
\end{proof}

Conjecture~\ref{thm: learning-lateral-3} says that with high probability, the learning algorithm yields a correct representation for $c$ and its descendants.
We can also consider correctness of a combination of successful learning of $c$ and successive recognition of $c$ after learning:

\begin{corollary}
\label{cor: combined-learning-and-recognition-lat}
Let $c$ be any particular concept in $\mathcal C$, with $level(c) = \ell, 1 \leq \ell \leq \ell_{max}$.
The learning algorithm for $c$, with probability at least $1 - \theta_{\ell}$, yields a network that satisfies the following property. 
Let $B \subseteq C_0$ such that $c$ is in $supp_{r_2}(B)$.
Then with probability at least $1 - \delta_{\ell}$, at least $m p(1-\zeta)$ of the neurons in $reps(c)$ fire at all times starting from some time $\geq 0$.
\end{corollary}

\begin{proof}
By Conjecture~\ref{thm: learning-lateral-3} and Theorem~\ref{thm: main-firing-lat-low}.
The first of these yields a correct representation with probability at least $1 - \theta_{\ell}$ and the second relies on the representation to achieve recognition with probability at least $1 - \delta_{\ell}$.
\end{proof}

As in Section~\ref{sec: learning-ff-low}, we can also consider a combined algorithm, one that first learns concept $c$, then attempts to recognize $c$ based on a presented set $B$ such that $c$ is in $supp_{r_2}(B)$.
The probability that recognition succeeds in this combined algorithm is at least 
$(1 - \theta_{\ell})(1 - \delta_{\ell})$.

\section{Conclusions}
\label{sec: conclusions}

\paragraph{Summary:}
We have described how hierarchical concepts can be represented in three types of layered neural networks, in such a way as to support recognition of the concepts when partial information about the concepts is presented, and also when some of the neurons might fail.  Our failure model involves only initial random failures.
The first two types of network are feed-forward, with high connectivity and low connectivity respectively.  
The third type of network also includes lateral edges, and has low connectivity; this case is inspired by prior work on the assembly calculus.

In order to achieve fault-tolerance, our representations all provide redundancy by including multiple $rep$ neurons for each concept.
Our representations are required to contain sufficiently many weight $1$ edges from $reps$ of child concepts to $reps$ of their parents.  Also, in networks with lateral edges, the representations are required to contain sufficiently many weight $1$ edges between $reps$ of the same concept.
The requirements for representations in this last case are embodied in a new assumption that we call the Class Assumption.

We have described how recognition works in all three of these settings, and have quantified how the probability of correct recognition depends on several parameters, including the number $m$ of $reps$ and the neuron failure probability $q$.
As one might expect, this probability increases with an increase in the number of $reps$ and decreases with an increase in the failure probability.  
The proofs use elementary probabilistic analysis, mainly Chernoff bounds and union bounds.

We have also discussed briefly how these representations might be learned, in all three types of networks.  For the feed-forward networks, the learning procedures are analogous to one used in~\cite{LM21}, whereas for networks with lateral edges, the procedure is generally inspired by one introduced in~\cite{DBLP:conf/innovations/Legenstein0PV18,PapadimitriouVempala}.

\paragraph{Discussion:}
The setting of this paper is fairly complicated, in that it includes three types of
limitations on information available for recognition.
First, we are trying to recognize concepts with only partial information, as in~\cite{LM21}. 
Second, we are dealing with random neuron failures.
Third, we are coping with partial connectivity between and within layers.
The combination of three types of partial information makes the setting quite tricky to understand and analyze.  
To cope with these difficulties, we have made many simplifications, in particular, by making strong uniformity assumptions, such as assuming that all concepts have the same number of child concepts and that all neurons have the same (independent) probability of failure.
Even so, the analysis is a bit tricky, and we expect that our bounds might not be the tightest possible.

We think that the ideas here for representing structured concepts are generally consistent with what is known about brain representations for such concepts, although of course drastically simplified.  One omission here is that, in representing a hierarchical structure, we consider only forward and lateral edges.  It would be reasonable to also include feedback edges.
This might help, for example, in recognizing a child concept based on first recognizing its parent.
We leave that complication for future work.  In~\cite{DBLP:conf/sirocco/LynchM23}, we studied recognition in networks with feedback edges, but in a non-fault-tolerant setting with single $reps$ for concepts.

As we noted, our work in Sections~\ref{sec: recog-lat-low} and~\ref{sec: learning-lateral} on networks with lateral edges was generally inspired by ideas from the assembly calculus~\cite{DBLP:conf/innovations/Legenstein0PV18,PapadimitriouVempala}.
However, we found the behavior of the learning algorithm in those papers difficult to understand, because of complications such as repeated changes in the selected $rep$ neurons.
We have proposed a simplified variant of the algorithm of~\cite{DBLP:conf/innovations/Legenstein0PV18,PapadimitriouVempala} in which the sets of chosen $reps$ do not change; we expect that will be easier to analyze, although we have not completed such an analysis.

\paragraph{Future work:}
Much remains to be done.
%
First, in general, it is worth simulating the algorithms in this paper, those for recognition as well as those for learning.
It is possible that the probabilities claimed in our results are not tight; simulations might yield better success probabilities, and might suggest how to improve the analysis.
This paper contains many parameters; simulation should help determine which combinations of parameter values yield high probability of correct recognition or correct learning.

For recognition, one might consider more complicated failures models than just random initial failures, such as randomly-occurring on-line failures.
It is probably not a good idea to consider worst-case, adversarial failures, which might depend on previous random and nondeterministic choices, since those would create dependencies that would be difficult to for an algorithm to cope with.
It is also worth trying to extend the recognition work to more elaborate network models that include some feedback edges from a layer to the next-lower layer.

As a technical idea, one way of understanding the behavior of a failure-prone network with multiple $reps$ per concepts might be to try to "map" it to a failure-free network with a single $rep$ per concept.
This might be done with a formal abstraction mapping of the types studied in~\cite{DBLP:series/synthesis/2010Kaynar} and~\cite{DBLP:journals/njc/SegalaL95}.

It remains to complete the studies of learning algorithms in Section~\ref{sec: learning}.
The algorithms in that section are based on earlier algorithms in~\cite{LM21,DBLP:conf/sirocco/LynchM23}, but use multiple $reps$ instead of a single $rep$ for each concept, and include considerations of failures.
It remains to rework the earlier presentations for the new settings.
While we expect that these extensions should be fairly routine, we could be surprised.
It remains to work them out carefully.


Section~\ref{sec: learning-lateral} will probably require more serious work.
It would be interesting to compare the algorithm in Section~\ref{sec:  new-learning-algorithm} to the algorithm in Section~\ref{sec: our-version-assembly-calculus}, first via simulation and then via analysis.
The first step would be to devise complete, formal descriptions of both algorithms.
In doing this, one might find that these learning algorithms do not quite satisfy the Class Assumption of Section~\ref{sec: representation-lat-low} as currently stated.  In this case, one would need to modify the assumption, show that the learning algorithms guarantee it, and show that the new assumption suffices for recognition.

Once these learning results are fully worked out, one can consider extensions, such as including feedback edges in the network.
One can also consider failures during the learning process.
As noted earlier, initial failures of layer $0$ neurons do not seem reasonable, because then we would entirely miss learning some parts of the concept hierarchy.  We might consider initial failures of higher-layer neurons, or failures of different layer $0$ neurons at the start of each new training instance; this latter type of failure is similar to what we called "noisy learning" in~\cite{LM21}.  
This all remains to be worked out.

\bibliography{Multi}

\appendix

\section{Chernoff bound}
\label{app: prob}

For Chernoff, we use a lower tail bound, in the following form:
\[
\text{For any } \zeta \in [0,1], \Pr[X \leq (1-\zeta) \mu] \leq \exp(-\frac{\mu \zeta^2}{2}).
\]
This is taken from the 2015 lecture notes for MIT course 18.310, by Michel Goemans~\cite{MG15}.

\end{document}